\documentclass[11pt]{article}
\input{header}

\renewcommand{\algorithmicrequire}{\textbf{Input:}}
\renewcommand{\algorithmicensure}{\textbf{Output:}}
\newcommand{\ajcomment}[1]{}

\makeatletter
\newcommand{\labitem}[2]{%
\def\@itemlabel{\text{#1}}
\item
\def\@currentlabel{#1}\label{#2}}
\makeatother

\begin{document}

\title{Learning Rate Schedules in the Presence of Distribution Shift}

\author[1]{Matthew Fahrbach}
\author[1,2]{Adel Javanmard}
\author[1]{Vahab Mirrokni}
\author[1]{Pratik Worah}
\affil[1]{Google Research, \texttt{\{fahrbach,mirrokni,pworah\}@google.com}}
\affil[2]{University of Southern California, \texttt{ajavanma@usc.edu}}
\date{}

\maketitle

\begin{abstract}
We design learning rate schedules that minimize regret for SGD-based online learning in the presence of a changing data distribution.
We fully characterize the optimal learning rate schedule for online linear regression via a novel analysis with stochastic differential equations.
For general convex loss functions, we propose new learning rate schedules that are robust to distribution shift,
and we give upper and lower bounds for the regret that only differ by constants.
For non-convex loss functions,
we define a notion of regret based on the gradient norm of the estimated models
and propose a learning schedule that minimizes an upper bound on the total expected regret.
Intuitively, one expects changing loss landscapes to require more exploration,
and we confirm that optimal learning rate schedules typically increase in the presence of distribution shift.
Finally, we provide experiments
for high-dimensional regression models and neural networks
to illustrate these
learning rate schedules and their cumulative regret.
\end{abstract}
\section{Introduction}

A fundamental question when training neural networks is how much of the weight space to explore and when to stop exploring.
For stochastic gradient descent (SGD)-based training algorithms, this is primarily governed by the learning rate.
If the learning rate is high, then we explore more of the weight space and vice versa.
Learning rates are typically decreased over time
in order to converge to a local optimum,
and there is now a substantial literature focused on how fast learning rates should decay for fixed data distributions
(see, e.g.,~\citet{pmlr-v75-tripuraneni18a} and \citet{JMLR:v19:17-370}, and the references therein).

However, what should we do if the data distribution is constantly changing?
This is the case in many {large-scale} online learning systems where
(1) the data arrives in a stream,
(2) the model continuously makes predictions and computes the loss, and
(3) it always updates its weights based on the new data it sees~\citep{anil2022factory}.
The goal of such a system is to always keep the loss low.
In this setting, convergence is less of a priority since
the model needs to be able to adapt to distribution shifts.
Intuitively, if the loss landscape is consistently changing
(either gradually or due to infrequent sudden spikes),
then it is sensible for the model to always explore its weight space.
We formalize this idea in our work.

Such an investigation naturally leads to the question of how high the learning rate should be,
and what an optimal learning rate schedule is in an online learning scenario?
These questions are critical because while tuning the learning rate can lead to improved accuracy in many applications,
it can also make the online learner widely inaccurate if
the wrong learning rate is used as the distribution changes.

\begin{figure*}
    \centering
    \includegraphics[width=0.24\textwidth]{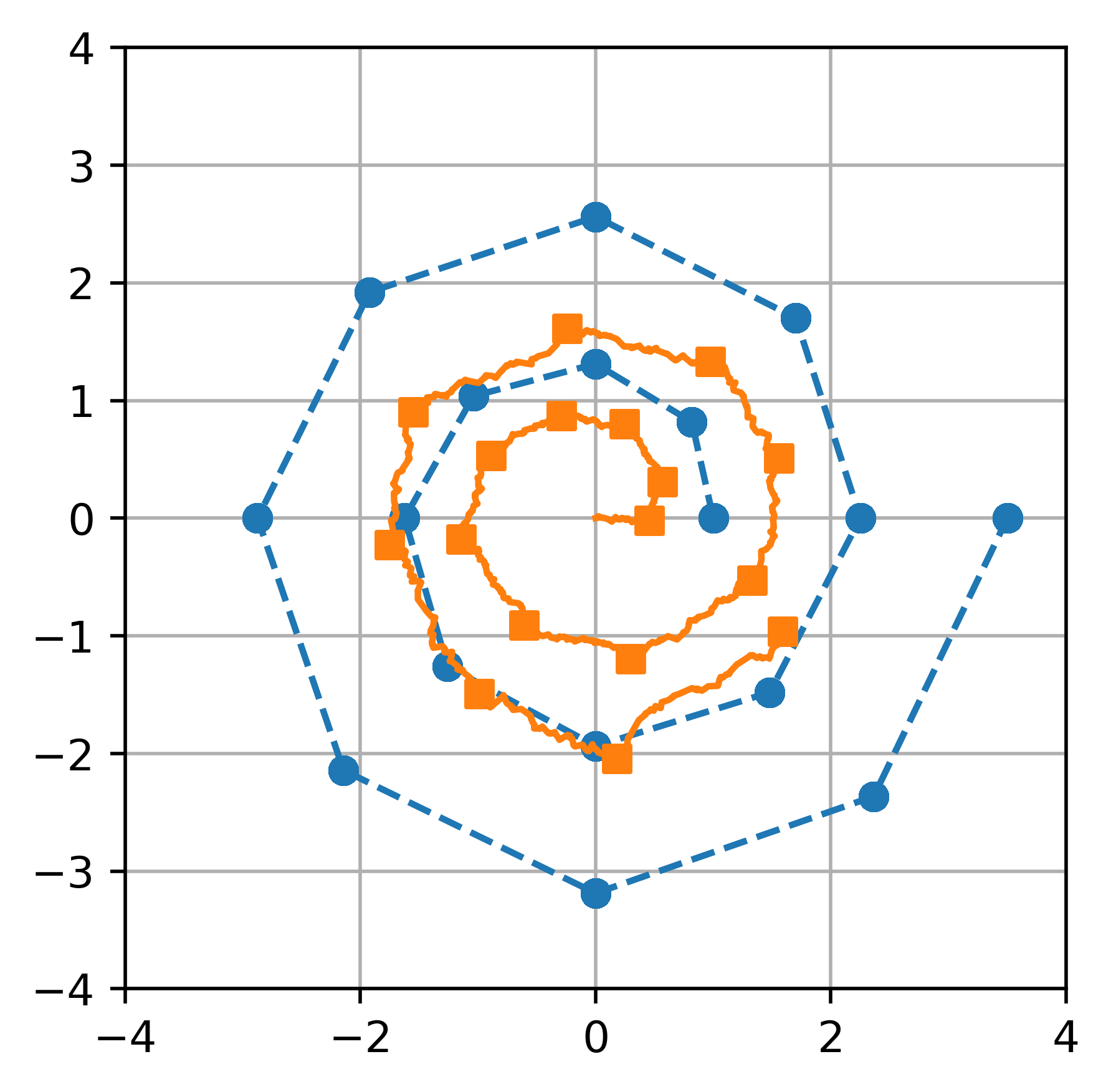}
    \hfill
    \includegraphics[width=0.24\textwidth]{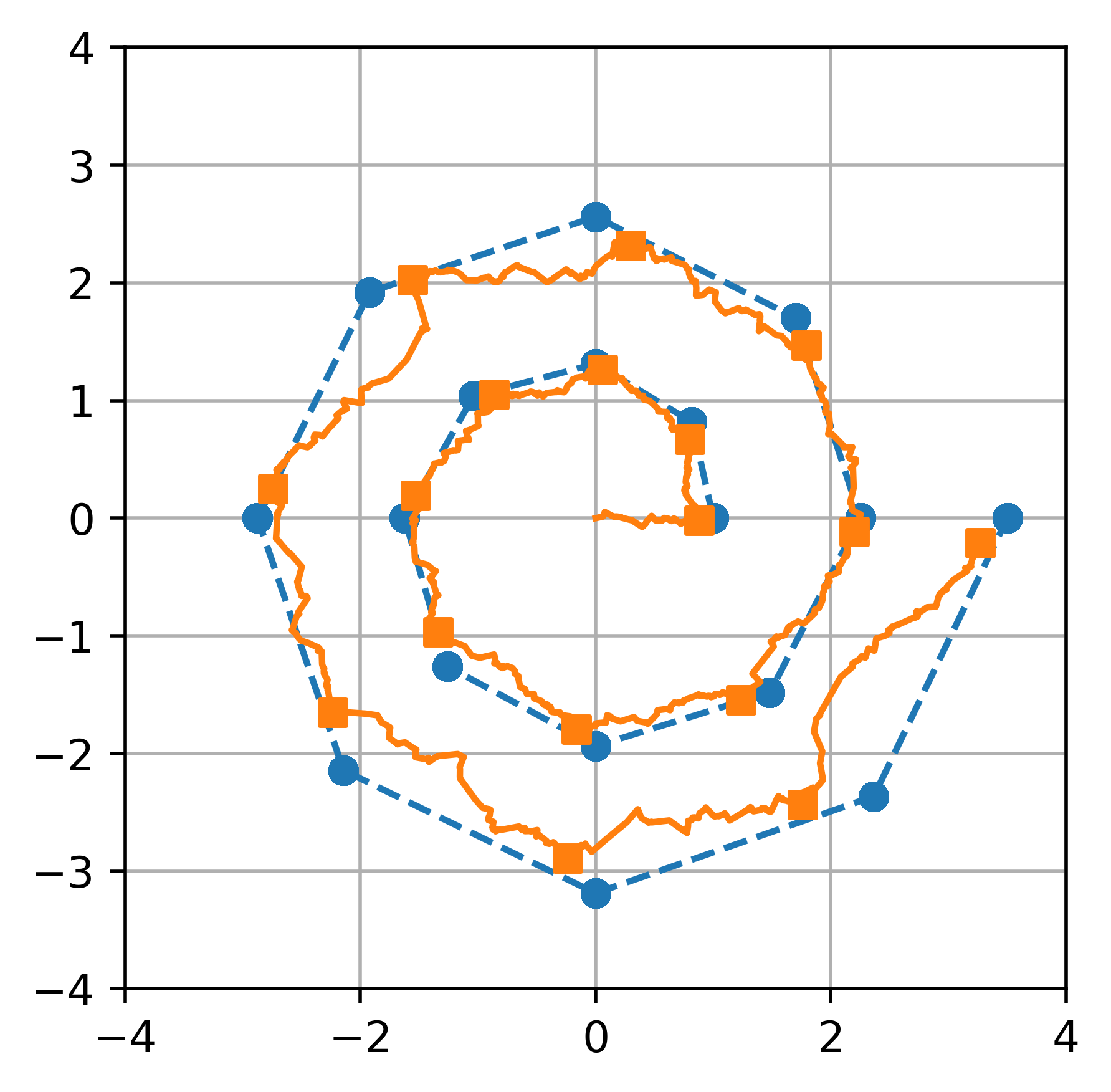}
    \hfill
    \includegraphics[width=0.24\textwidth]{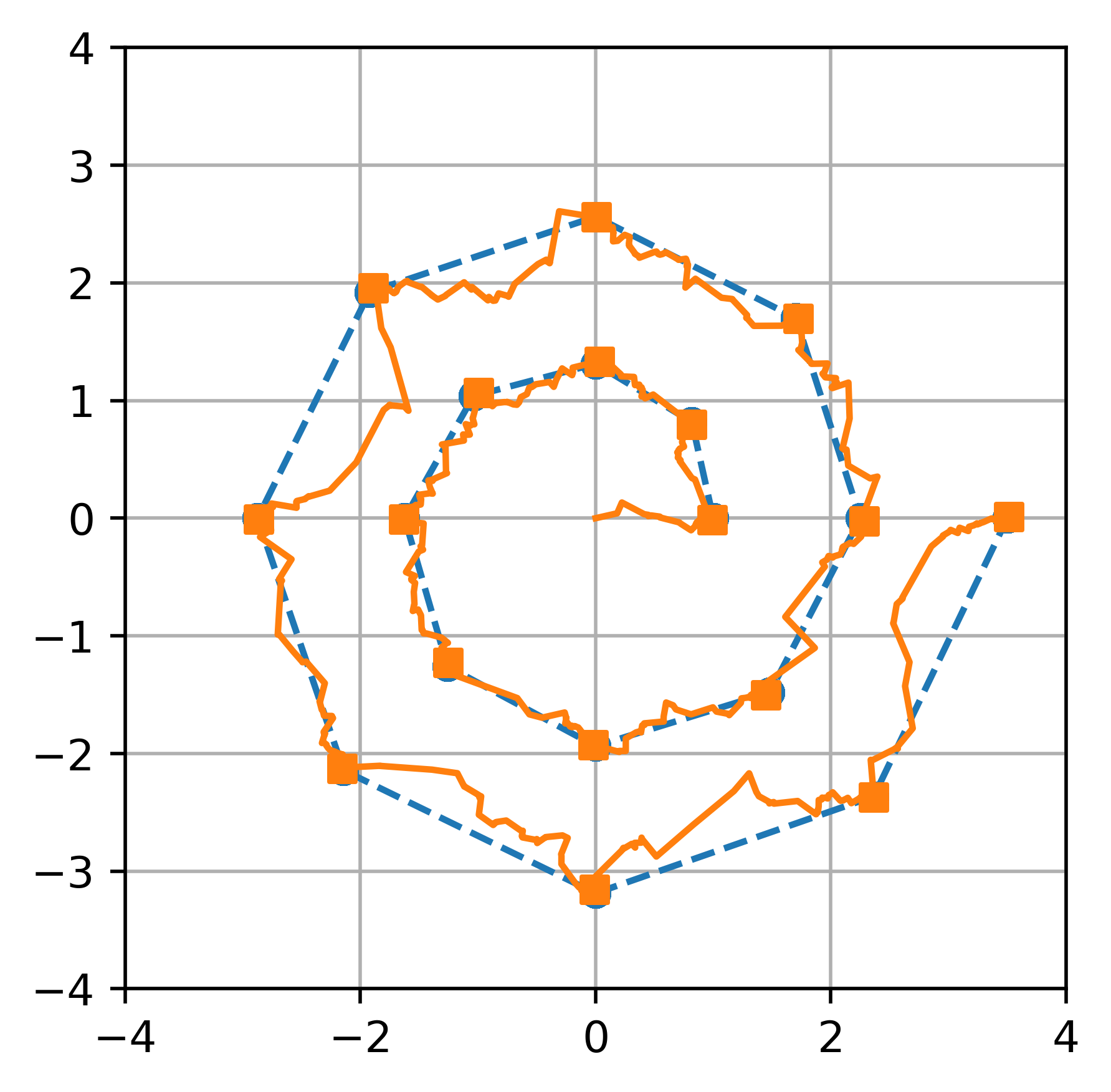}
    \hfill
    \includegraphics[width=0.24\textwidth]{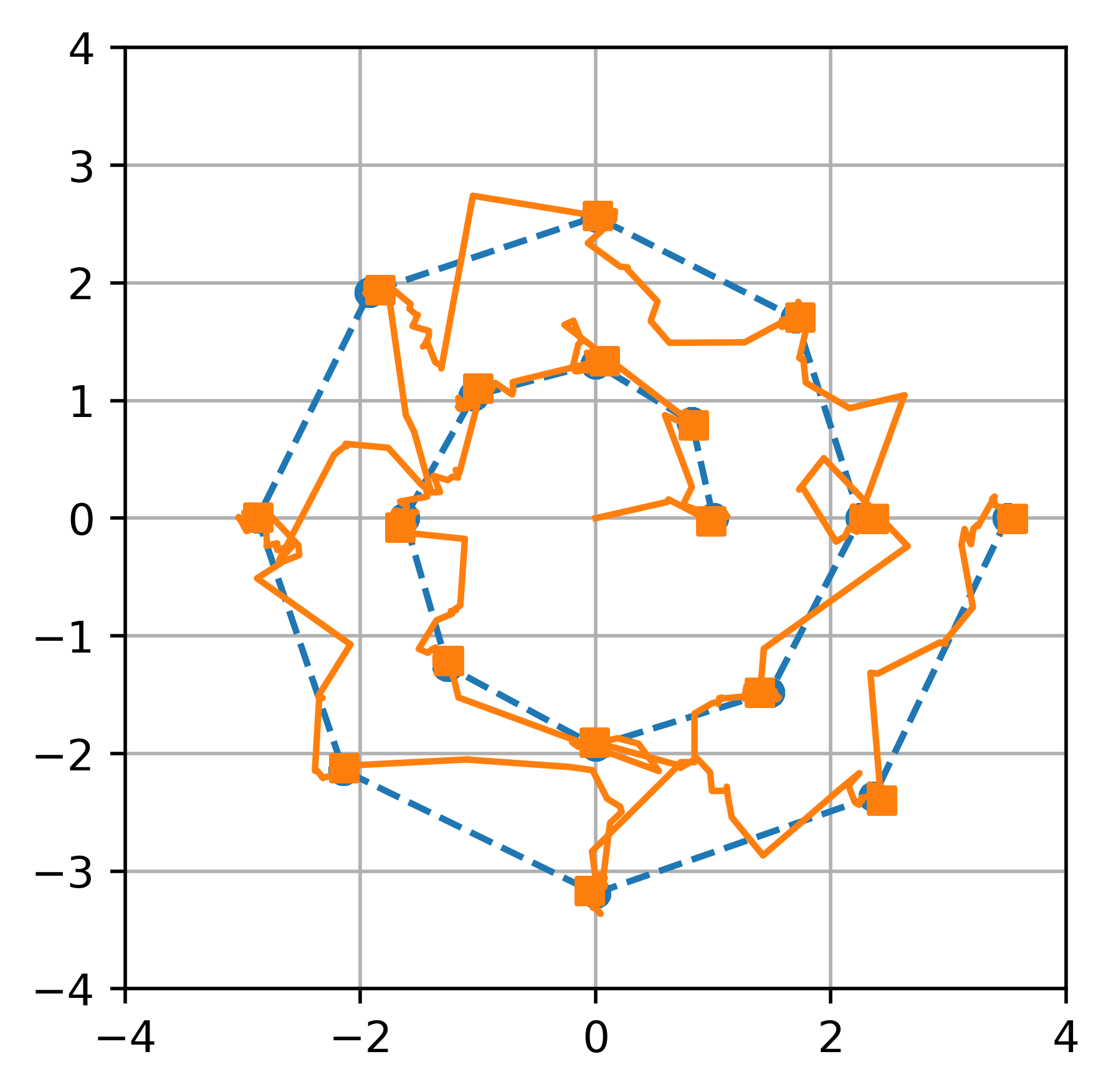}
    \caption{SGD trajectories for online linear regression with different constant learning rates.
    The discrete blue spirals are the optimal model weights $\theta_{t}^* \in \mathbb{R}^{2}$,
    which start at $(1,0)$ and jump clockwise
    every $100$ steps.
    The orange paths are the learned weights $\theta_t$,
    starting at $\theta_0 = 0$
    for $0 \le t \le 17 \cdot 100$.
    The orange squares depict the position every $100$ steps.
    We use batch size $B_t = 1$
    and step sizes
    $\eta_{t} \in \{0.003, 0.01, 0.03, 0.1\}$ from left to right.
    The rightmost SGD is the most out of control,
    but it incurs the least regret
    because it adapts to changes in $\theta_{t}^*$ the fastest
    without diverging.
    }
    \label{fig:warmup_spirals}
    \vspace{-0.1cm}
\end{figure*}

Formally, we study learning rate schedules in the presence of distribution shifts by considering \emph{dynamic regret}, a well-known notion in online optimization that measures the performance against a dynamic comparator sequence.
This regret framework captures the lifetime performance of an online learning system that makes predictions on incoming examples as they arrive (possibly from a time-varying distribution)
before using this data to update its weights. 
    
Our main contributions can be summarized as follows:

%\begin{enumerate}
%\item \textbf{Linear regression.}
\paragraph{Linear regression.}
We consider a linear regression setup with time-varying coefficients $\{\theta^*_t\}_{t \ge 1}$, which are chosen upfront by an adversary such that $\|\theta^*_t-\theta^*_{t+1}\|_{2} \le \gamma_t$ for a sequence of positive numbers $\{\gamma_t\}_{t\ge 1}$.  The variation in the model coefficients results in response shift (while the covariates distribution remains the same across time). We consider a learner who updates their model estimates via mini-batch SGD with an adaptive step size sequence $\{\eta_t\}_{t\ge 1}$ chosen in an online manner (i.e., only with access to previous data points). We derive a novel stochastic differential equation (SDE) that approximates the dynamics of SGD under distribution shift, and by analyzing it, we derive the optimal learning rate schedule.

%\item \textbf{Convex loss functions.}
\paragraph{Convex loss functions.}
We generalize our problem formulation along the following directions: $(i)$ We consider general convex loss functions $\ell(\th,z)$ that measure the loss of a model $\th\in\reals^p$ on the data point $z\in\reals^d$. $(ii)$ At each step the learner observes a batch of data points $\{z_{t,k}\}$ drawn from a time-varying distribution $P_t$, which means it can model both response shift and covariate shift. $(iii)$ An adversary can choose the distributions $P_t$ adaptively at each step by observing the history (i.e., the data and model estimates from previous rounds), in contrast to the linear regression setup where the sequence of models are time-varying but fixed a priori.  For strongly convex loss functions, we give a lower bound for the total expected regret that is of the same form as our upper bound and differs only in the constants, demonstrating that our regret analysis is nearly tight. We then propose a learning rate schedule to minimize the derived upper bound on the regret. This schedule is adaptive, resulting in a time-dependent learning rate that tries to catch up with the amount of distribution shift in the moment. We refer to Section~\ref{sub:lit} for a  detailed comparison to the literature on online convex optimization in dynamic environments.
    
%\item \textbf{Non-convex loss functions.}
\paragraph{Non-convex loss functions.}
For settings with non-convex loss functions, we modify the notion of regret to use the gradient norm of the estimated model. We derive an upper bound for the expected cumulative regret and propose a learning rate schedule that minimizes it.
In our experiments in \Cref{app:cyto},
we use neural networks and dynamic learning rates to
continuously classify cells arriving in a stream of
small condition RNA data \citep{Bastdas-Ponce}.
This work simulates an online and deep learning-based
\emph{flow cytometry} algorithm.
We refer the reader to \citet{LMC19} for more details about this application.
One take-away message from our analysis and experiments
in all three settings is that an optimal learning rate schedule
typically increases in the presence of distribution shift. 

The organization of the paper is as follows.
In \Cref{sub:lit}, we proceed with a literature review.
In \Cref{sub:res}, we present an overview of our tools,
techniques, and informal statements of our theoretical results.
We formally define the problem in \Cref{sec:formulation}.
We present our results for linear regression in \Cref{sec:linear-regression},
convex losses in \Cref{sec:cvx}, and non-convex losses in \Cref{sec:nonconvex}.
In \Cref{sec:experiments},
we present experiments to study the effect
of the proposed learning rate schedules,
including high-dimensional regression and a
medical application to flow cytometry.
We defer the proofs of our technical results to the appendix.

\subsection{Related work}
\label{sub:lit}
With deep neural networks now being used in countless applications and SGD remaining the dominant algorithm for training these models, there has been a surge of effort to understand how learning rates affect the behavior of stochastic optimization methods~\citep{Bengio,Smith}. Most of the existing literature, however, assumes no shift in the underlying distribution across the iterations of SGD. Various trade-offs between learning rate and batch size have been studied~\citep{keskar2016large,smith2018don}. In particular, \citet{smith2018don} propose that instead of the decaying learning rate, one can increase the batch size during training and empirically show that it results in near-identical model performance with significantly fewer parameter updates.
\citet{shi2020learning} analyze the effect of learning rate on SGD by studying its continuum formulation given by a stochastic differential equation (SDE) and show that for a broad class of losses, this SDE converges to its stationary distribution
at a linear rate, further revealing the dependence of a linear convergence rate on the learning rate. Learning rate schedules for SGD, under fixed distribution, and for the setting of least squares has been studied in~\citep{ge2019step,jain2019making}.
Decaying learning rate via cyclical schedules has also
been proposed for training deep neural models (see, e.g.,~\citet{Smith,loshchilov2016sgdr,li2019exponential}).

The effects of SGD hyperparameters (e.g., batch size and learning rate) have also been studied for the adversarial robustness of the resulting models \citep{yao2018hessian,kamath2020sgd}.
In this setting, a model is trained on unperturbed samples, but at test time the sample features are slightly perturbed. In contrast, this paper considers settings where \emph{the data distribution is constantly changing}---even during training---and studies the effect of learning rates in presence of such distribution shifts.

\paragraph{Connections to online optimization.}
The notion of dynamic regret has been used in online convex optimization to evaluate the performance of a learner against a dynamic target, as opposed to the classical single best action in hindsight~\citep{zinkevich2003online,yang2016tracking,jadbabaie2015online,besbes2015non,bedi2018tracking}. In this setting, nature chooses a sequence of convex functions $f_1,f_2,\dots, f_T$ and the learner chooses a model (i.e, action) $\theta_t$ at each step and incurs loss~$f_t(\theta_t)$.
Our problem is closest to the works of~\citet{besbes2015non} and \citet{bedi2018tracking}, in which the learner only has noisy access to gradients $\nabla f_t(\theta_t)$. There is often a notion of variation to capture the change in the comparator.
For example,~\citet{yang2016tracking} consider ``path variation'', which measures how fast the minimizers of the sequence of loss
functions change;
\citet{besbes2015non} defines a ``functional variation''
based on the supremum distance between consecutive loss functions;
and \citet{bedi2018tracking} track the ``path length'' between minimizers
(i.e., what we call \emph{distribution shift} in this work).

\citet{yang2016tracking} bound the cumulative dynamic regret when a constant step size $\eta\propto\sqrt{\mathcal{V}_T/T}$ is used, where $T$ is the horizon length and $\mathcal{V}_T$ is the variation budget that controls the power that nature has in choosing the sequence of loss functions (see Theorem~7 therein).
\citet{besbes2015non} propose a restarting procedure, which for batch size $\Delta_T$ restarts an online gradient descent algorithm every $\Delta_T$ periods. Their analysis suggests to take $\Delta_T = (T/\mathcal{V}_T)^{2/3}$ and $\eta \propto 1/\sqrt{\Delta_T}$ (see Theorem 3 therein).
\citet{bedi2018tracking} design and analyze the
inexact online gradient descent (IOGD) algorithm.

While these works also suggest that in a changing environment the learning rate should in general be set higher,
our formulation and analysis for the convex setting departs from these works in the following ways: $(i)$ Instead of constant or a pre-determined learning rate, our framework allows for \emph{adaptive schedules} where the learning rate at each epoch can be set based on the history; $(ii)$ The notion of dynamic regret is often defined with respect to an arbitrary but fixed sequence of loss functions satisfying a variation budget. In contrast, we allow the data distribution to shift adaptively at each step after observing the history, and so the expected loss changes adaptively over time;
$(iii)$ \citet{besbes2015non} and \citet{yang2016tracking} establish lower bounds on the dynamic regret, but these bounds are for the worst-case regret over the choice of loss function sequences that satisfy the variation budget. These lower bounds are obtained by carefully crafting a sequence that is hard to optimize in an online manner. However, there is a subtle difference in our setting---the loss function $\ell(\theta,z)$ is fixed and the change in expected loss over time comes from a shift in the data distribution~$z$.
The lower bound we derive for dynamic regret assumes
the same \emph{fixed loss function} $\ell(\theta,z)$.

\subsection{Overview of  techniques}
\label{sub:res}

To analyze the behavior of SGD in a linear regression setting,
we derive a novel stochastic differential equation (SDE) that approximates the dynamics of SGD in the presence of distribution shift. Using Gr\"{o}nwall's inequality~\citep{gronwall1919note}, we control the deviation of the SGD trajectory from the continuous process and relate the regret of SGD to the second moment of the continuous process, which we characterize using the celebrated It\^o's lemma from stochastic calculus~\citep{oksendal2013stochastic} (see Lemma~\ref{lem:Ito}).
Using this characterization, we derive an
optimal learning rate schedule in a sequential manner.

Our results for general convex loss functions are based on an intricate treatment of the regret terms, taking the expectation with respect to a proper filtration and applying several properties of convex functions and SGD  itself.

Non-convex loss functions can have a complicated landscape with potentially many local minima and saddle points. Even without distribution shifts, first-order methods like SGD are not guaranteed to converge to a global minimum. To deal with this, we modify the definition of regret to use the norm of the gradient of the loss for the estimated models. Thus, a trajectory that stays close to local minima of the loss functions has low total regret. To upper bound the cumulative regret in this setting, we follow a similar proof technique as in the convex case, but rely only on the SGD update formulation and first-order optimality conditions on the sequence of optimal weights $\{\theta^*_t\}_{t \ge 1}$.
%\section{Problem formulation: A regret minimization framework}
\section{Problem formulation: Dynamic regret}
\label{sec:formulation}

We consider an online sequential learning setting where at each step $t$ the learner observes a batch of size $B_t$ data points $\bz_t = \{z_{t,k}\}_{k=1}^{B_t}$ drawn independently from a distribution $P_t$. The distributions $P_t$ can vary with time and are defined on $\reals^d$. 
The batch loss incurred at step $t$ is $\frac{1}{B_t}\sum_{k=1}^{B_t} \ell(\theta_t,z_{t,k})$ for a function $\ell:\reals^p\times \reals^d\to \reals_{\ge 0}$. 
The learner then updates its model weights $\theta_t \to \theta_{t+1}\in\reals^{p}$.

Define the expected loss as $\ellbar_t(\theta) := \E_{P_t}[\ell(\theta,z_{t,k})]$.
Letting $(\theta_1,\theta_2,\dotsc)$ denote the sequence of learned models, the total expected loss up to time $T$ is $\sum_{t=1}^{T} \ellbar_t(\theta_t)$.
The goal of the learner is to minimize the above objective.
For each step $t$, we define an \emph{oracle model} with weights
\begin{equation}
\label{eqn:theta_star}
    \theta^*_t := \argmin_{\theta \in \mathbb{R}^{p}} \ellbar_t(\theta).
\end{equation}
Since the distributions $P_t$ can vary with time,
the weights~$\theta^*_t$ also shift over time.

Instead of minimizing the total loss, we equivalently work with the \emph{regret} of the learner defined
with respect to the comparator sequence $(\theta^*_1, \theta^*_2,\dotsc)$ below:
\begin{align}
\label{eqn:regret_def}
    \Reg(T)
    :=
    \sum_{t=1}^{T} \reg_{t}
    ,\quad
    \reg_{t} := \ellbar_t(\theta_t) - \ellbar_t(\theta_t^*).
\end{align}
Note that $\Reg(t)$ and $\reg_{t}$ are random variables
that depend on $\theta_t$. This framework can be seen
as a game between nature, who chooses the distributions $P_t$
(and thus the sequence of oracle models~$\theta^*_t$),
and the learner, who must choose
the sequence of models $\theta_t$ for $t\ge 1$.

The learner updates its weights using projected mini-batch stochastic gradient descent (mini-batch SGD) given by
\begin{align}\label{eq:BSGD}
    \theta_{t+1} &= \Pi_{\Theta}\left(\theta_t  - \eta_t \nabla\ell^B_t(\theta_t)\right) \\
    \nabla\ell^B_t(\theta_t) &:=
    \frac{1}{B_t} \sum_{k=1}^{B_t} \nabla\ell(\theta_t,z_{t,k})\,,
\end{align}
where $\nabla\ell(\theta_k,z_k)$ are stochastic gradients, $\Theta$ is a bounded convex set, and $\Pi_{\Theta}$ is the projection onto the admissible weight set $\Theta \subseteq \mathbb{R}^{p}$.
Observe that $\E[\nabla\ell(\theta_k,z_k)] = \nabla \ellbar(\theta_k)$, and therefore the sample average gradient above is an unbiased estimate of the gradient of the expected loss.

Nature is allowed to be \emph{adaptive} in that she can set $\theta^*_t$ after observing the
history of the data
\begin{align}\label{eq:history}
 \bz_{[t-1]} := \{ (z_{k,1}, z_{k,2}, \dots, z_{k,B_k})  : 1 \le k \le t-1\}.
\end{align}
The step sizes $\eta_t$, called the \emph{learning rate schedule},
can also change over time in an adaptive manner, i.e.,
the learning rate $\eta_t$ is a function of $\bz_{[t-1]}$.
Note that by the SGD update, $\theta_t$ is a function of $\bz_{[t-1]}$, and so $\theta^*_t$ can depend on the previously learned models $\theta_s$
for $s < t$.
The learning rate schedule controls how the step size changes across iterations.

\begin{definition}[Distribution shift]
\label{def:shift}
Recall the definition of oracle models $\theta_{t}^*$ in \eqref{eqn:theta_star}.
We quantify the \emph{distribution shift} (variation of $P_t$ over time) in terms of the variation in oracle models, namely
\begin{equation}
\label{eqn:gamma_t_def}
    \gamma_t: = \|\theta^*_t - \theta^*_{t+1}\|_{2}\,.
\end{equation}
\end{definition}

\noindent
This is related to the notion of path length between minimizers
in~\citet{bedi2018tracking}.
\section{Linear regression}
\label{sec:linear-regression}

We start by studying the linear regression setting with a time-varying coefficient model~\citep{fan2008statistical,hastie1993varying}.
Each sample $z_{t,k} = (x_{t,k},y_{t,k})$ is a pair of covariates $x_{t,k}\in\reals^d$ and a label $y_{t,k}$, with 
\begin{align}\label{eq:LinModel}
    y_{t,k} = \<x_{t,k},\theta^*_t\>+ \eps_{t,k}\,,
\end{align}
where $\eps_{t,k}\sim\normal(0,\sigma^2)$ is random noise. The covariate distribution is assumed to be the same across time, and for simplicity assumed as $x_{t,k}\sim\normal(0,I)$. The model $\theta^*_t$ changes over time, so we have label distribution shift.
We consider least squares loss $\ell(\theta,z) = \frac{1}{2}(y - \<x,\theta\>)^2$, for $z = (x,y)$. 

% We consider SGD updates~\eqref{eq:BSGD} (with $\Theta = \reals^d$, i.e., no projection):
% \begin{align}
%     \theta_{t+1} &= \theta_t - \eta_t \frac{1}{B_t}\sum_{k=1}^{B_t}\nabla\ell(\theta_t,z_{k,t})\nonumber\\
%     &= \theta_t + \eta_t \frac{1}{B_t}\sum_{k=1}^{B_t} (y_{tk}-\<x_{tk},\theta\>) x_{tk}\nonumber\\
%     %&= \theta_t -\eta_t \nabla\ellbar_t(\theta_t)-
%     %\eta_t \left(\frac{1}{B_t}\sum_{k=1}^{B_t}\nabla\ell(\theta_t,z_{k,t}) - \nabla\ellbar_t(\theta_t)\right)\nonumber\\
%      &= \theta_t +\eta_t (\theta^*_t - \theta_t) +\frac{\eta_t}{B_t}\sum_{k=1}^{B_t} \left((x_{tk}x_{tk}^\sT - I)(\theta^*_t - \theta_t)+\eps_{tk} x_{tk}\right)\nonumber\\
%      & =\theta_t +\eta_t (\theta^*_t - \theta_t) -\eta_t\zeta_t \,, \label{eq:zeta}
% \end{align}
% where the noise term $\zeta_t$ has mean zero, given that the data points $z_{t,k}$ are sampled independently at each step $t$. 

To provide theoretical insight on the dependence of SGD on the learning rate under distribution shift, we follow a recent line of work that studies optimization algorithms via the analysis of their behavior in continuous-time
limits~\citep{krichene2017acceleration,li2017stochastic,chaudhari2018deep,shi2020learning}.
Specifically, for SGD this amounts to studying stochastic differential equations (SDEs) as an approximation for discrete stochastic updates. The construction of this correspondence is based on the Euler--Maruyama method. We assume that the step sizes in SGD are given by $\eta_t = \eps \zeta(\eps t)$, where $\zeta(t)\in [0,1]$ is the adjustment factor and $\eps$ is the maximum allowed learning rate. In addition, the batch sizes are given by $B_t = \eps \nu(\eps t)$, for sufficiently regular functions $\zeta, \nu:\reals_{\ge 0}\to \reals_{\ge0}$.\footnote{More precisely, $B_t = \lceil\eps \nu(\eps t)\rceil$ must be an integer,
however, the rounding effect is negligible in the continuous time analysis.}  

We use $t$ to denote the iteration number of SGD and use~$\tau$ as the continuous time variable for the corresponding SDE.
We show that the trajectory of SGD updates can be approximated by the solution of the following SDE:
\begin{align}
\label{eq:SDE}
    \de X(\tau) =-(\zeta(\tau) X(\tau) +Y'(\tau))\de\tau
    + \tfrac{\zeta(\tau)}{\sqrt{\nu(\tau)}} \left((\|X(\tau)\|^2+\sigma^2)I + X(\tau)X(\tau)^\sT\right)^{1/2}\; \de W(\tau)\,,
\end{align}
where $X(0) = \theta_0 - \theta^*_0$
and $Y(\tau)$ is a sufficiently smooth curve so that $Y(\eps t) = \theta^*_t$. Further, $W(\tau)$ is $d$-dimensional vector with each entry being a standard Brownian motion, independent from other entries. To make this connection, we posit the following assumptions:
\begin{itemize}
    \item[{\sf A1.}] The functions $\zeta(\tau)$ and $\zeta(\tau)/\sqrt{\nu(\tau)}$ are bounded Lipschitz:
    $\|\zeta\|_\infty$, $\|\zeta\|_{\rm Lip}$,
    $\|\zeta/\sqrt{\nu}\|_\infty$, $\|\zeta/\sqrt{\nu}\|_{\rm Lip}\le K$.
    \item[{\sf A2.}] The function $Y(\tau)$ is bounded Lipschitz: $\|Y(\tau)\|\le K$ and $\|Y'(\tau)\|\le \Gamma/\eps$, for constants $K, \Gamma>0$. Recall that $Y(\tau)$ is the continuous interpolation of the sequence models $\theta^*_t$ and therefore $Y'(\tau)$ controls how fast $\theta^*_t$ are changing and is a measure of distribution shift in the response variable $y_{tk}$ in~\eqref{eq:LinModel}.
\end{itemize}

In {\sf (A1)} we use the notation $\|f\|_{\rm Lip} := \sup_{x\neq y} |f(x)-f(y)|/|x-y|$ to indicate the Lipschitz norm of a function and $\|f\|_\infty:= \sup_x |f(x)|$.  
% Recall that by construction $Y'(t\eps)\approx (\theta^*_{(t+1)\eps}-\theta^*_{t\eps})/\eps$, and so the term $\gamma_{t\eps}$ controls the shifts in the model's $\theta^*_t$ over time.

\begin{propo}\label{pro:approx}
For any fixed $T,u>0$, there exists a constant $C = C(K,\Gamma, d,\sigma, T,u)$, with parameters $K,\Gamma$ given in Assumptions {\sf A1-A2}, such that with probability at least $1-e^{-u^2}$ we have
\[
    \sup_{t\in[0,T/\eps]\cap \mathbb{Z}_{\ge 0}} \Big|\|X_{t\eps}\|^2-\|\theta_t-\theta^*_t\|^2\Big|
    \le C\sqrt{\eps}\,.
\]
\end{propo}
We defer the proof of this proposition and the exact expression
for the constant $C$ to \Cref{proof:approx}.

\noindent
The expected regret at time $t$ works out as:
\begin{align*}
    \E[\reg_t] &= \E[\ellbar_t(\theta_t)- \ellbar_t(\theta^*_t)] \\
    &= \frac{1}{2}\E[(\<x_{tk},\theta_t-\theta^*_t\>+\eps_{tk})^2] - \frac{1}{2}\E[\eps_{tk}^2] \\
    &=\frac{1}{2}\E[\|\theta_t - \theta^*_t\|^2]\,.
\end{align*}
 Using Proposition~\ref{pro:approx}, $|\E[\reg_t] - \frac{1}{2} \E[\|X(t\eps)\|^2]|\le C\sqrt{\eps}$. Henceforth, we focus on analyzing the second moment of the process $X$, as $\eps$ can be fixed to an arbitrarily small value.

For $X(\tau)$ the solution of SDE~\eqref{eq:SDE}, we define
\begin{align}\label{eq:moments}
m_\tau := \E[X(\tau)]\in\reals^d, \quad
v_{\tau}: = \E[\|X(\tau)\|^2]\,.
\end{align}
In our next theorem, we derive an ODE for $m_\tau$ and $v_\tau$, using It\^o's lemma from stochastic calculus~\citep{oksendal2013stochastic}.
The proof is deferred to Section~\ref{sec:Ito-SDE}.
\begin{thm}\label{pro:SDE}
Consider the SDE problem~\eqref{eq:SDE}, and the moments $m_\tau$ and
$v_{\tau}$ given by~\eqref{eq:moments}. We have
\begin{align}
    m'_\tau &= -\zeta(\tau)m_\tau - Y'(\tau)\,,\label{eq:m-ode} \\
    v'_\tau &=
    \Big((d+1)\frac{\zeta(\tau)^2 }{\nu(\tau)}-2\zeta(\tau)\Big) v_\tau  
    + \frac{\zeta(\tau)^2}{\nu(\tau)}\sigma^2 d - 2m_\tau^\sT Y'(\tau)\,.  \label{eq:v-ode}
    %-2  e^{-\tau} X(0)^\sT \nabla Y(\tau)\,.
\end{align}
\end{thm}

% The regret $v_\tau$ can be solved in closed form from the differential equations~\eqref{eq:m-ode}--\eqref{eq:v-ode}. 

It is worth noting that from the above ODEs, larger distribution shift (quantified by the $Y'(\tau)$ term) increases the drift in $m_\tau$ as well as the drift in $v_\tau$ via the term $m_\tau^\sT Y'(\tau)$.
In this case, the learner needs to choose a larger step size $\zeta(\tau)$ to reduce the drift in $m_\tau$,
which is consistent with our message that in dynamic environments
the learning rate should often be set higher. 

The problem of finding an optimal learning rate can be
seen as an optimal control problem,
where the state of the system $(m_\tau,v_\tau)$ evolves according to ODEs~\eqref{eq:m-ode}--\eqref{eq:v-ode},
the control variables~$\zeta$ can take values in the set of   Borel-measurable functions from $[0,T]$ to $[0,1]$,
and the goal is to minimize the cost functional $\int_0^{T}v_\tau \de \tau$.
The optimal learning rate schedule can then be solved exactly by dynamic programming, using
the Hamilton--Jacobi--Bellman equation~\citep{bellman1956dynamic}.
However, the optimal learning rate will depends on $Y'(\tau)$, which is a $d$-dimensional time-varying vector.
We next do a simplification to reduce the dependence to $\|Y'(\tau)\|$. 

Note that $|m_\tau^\sT Y'(\tau)|\le \|Y'(\tau)\|\;\|m_\tau\|\le \|Y'(\tau)\| \sqrt{v_\tau}$. The first inequality becomes tight if the shift $Y'(\tau)$ is aligned with the expected error $m_\tau$. The second inequality becomes tighter as the batch size grows, since it reduces the variance in $X(\tau)$, which by~\eqref{eq:moments} is given by $v_\tau-\|m_\tau\|^2$. Therefore, we have
\[
v'_\tau \le
    \Big((d+1)\frac{\zeta(\tau)^2 }{\nu(\tau)}-2\zeta(\tau)\Big) v_\tau 
     + \frac{\zeta(\tau)^2}{\nu(\tau)}\sigma^2 d + 2\|Y'(\tau)\|\sqrt{v_\tau}.
\]
With this observation and the fact that
our objective is to minimize the cost $\int_0^T v_\tau\de\tau$,
we consider the process $\tilde{v}_\tau$ defined using the upper bound on $v_\tau'$, namely
\begin{align}\label{eq:tv}
    \tilde{v}'_\tau &=
    \Big((d+1)\frac{\zeta(\tau)^2 }{\nu(\tau)}-2\zeta(\tau)\Big) \tilde{v}_\tau
     + \frac{\zeta(\tau)^2}{\nu(\tau)}\sigma^2 d + 2\|Y'(\tau)\|\sqrt{\tilde{v}_\tau}\,.
\end{align}
Our next result characterizes an optimal learning rate with respect to process $\tilde{v}_\tau$.
\begin{thm}\label{thm:optimal-policy}
Consider the control problem
\[
\underset{\zeta:[0,T]\to [0,1]}{{\rm minimize}}\; \int_0^T \tilde{v}_\tau\de \tau\,,\quad \text{\emph{ subject to constraint}}~\eqref{eq:tv}\,.
\]
The optimal policy $\zeta$ is given by 
\begin{align}\label{eq:policy}
    \zeta^*(\tau) =\min\left\{1, \left(\frac{d+1}{\nu(\tau)}\tilde{v}_\tau + \frac{\sigma^2d}{\nu(\tau)}\right)^{-1} \tilde{v}_\tau\right\}\,.
\end{align}
\end{thm}

Using the policy $\zeta^*(\tau)$ given by~\eqref{eq:policy} and~\eqref{eq:tv},
we get an ODE that can be solved for $v_\tau$
and then plugged back into~\eqref{eq:policy} to obtain an optimal policy $\zeta^*(\tau)$ and hence optimal learning rate.
We formalize this approach in Algorithm~\ref{alg:linear}, where we solve the ODE for $\tilde{v}_\tau$ (after substituting for optimal $\zeta^*(\tau)$) using the (forward) Euler method.
Translating from the continuous domain to the discrete domain,
we use the relations $\eta_t = \eps\zeta(\eps t)$, $B_t = \eps\nu(\eps t)$, and $\|Y'(\eps t)\|\approx \|\theta^*_{t+1}-\theta^*_t\|/\eps = \gamma_t/\eps$.

\begin{remark}
The learning rate schedule proposed in Algorithm~\ref{alg:linear} is an online schedule in the sense that $\eta_t$ is determined based on the history up to time $t$, i.e., it does does not look into future.
%\aj{I changed the $\gamma_t$ to $\gamma_{t-1}$ in the algorithm since this is what we said to do in simulations, but didn't explain it here as it cause confusion.}
\end{remark}

\begin{remark}
The proposed learning rate in Algorithm~\ref{alg:linear} depends on the distribution shifts $\gamma_t$. In settings where $\gamma_t$ is not revealed (even after the learner proceeds to the next round),
we estimate $\gamma_t$ using an exponential moving average of the drifts in the consecutive estimated models $\theta_t$,
namely $\hat{\gamma}_t = \beta \hat{\gamma}_{t-1} + (1-\beta)\|\theta_t- \theta_{t-1}\|$, with a factor $\beta\in(0,1)$.
\end{remark}

\begin{algorithm}[tb]
   \caption{Optimal learning rate schedule for linear regression undergoing distribution shift.}
   \label{alg:linear}
\begin{algorithmic}
   \STATE {\bfseries Input:} max step size $\eps$,
                            discretization scale $\kappa\in (0,1]$
   \STATE{\bfseries Output:} step sizes $\eta^*_t$
   \STATE{\bfseries Initialization:} $v \gets 0$
   \FOR{$t = 1, 2, \dots$}
   \FOR{$j = 1, 2, \dots, \lceil 1/\kappa \rceil$}
   \STATE $r\leftarrow \min\Big(\frac{v B_t}{(d+1)v + \sigma^2 d},\eps \Big)$\;
   \STATE $v\leftarrow v+\kappa\Big(\frac{d+1}{B_t}r^2 - 2r\Big)v + \kappa \frac{\sigma^2d}{B_t} r^2+2\kappa \gamma_{t-1}\sqrt{v}$\;
   \ENDFOR
   \STATE $\eta^*_{t}\leftarrow\min\Big(\frac{v B_{t}}{(d+1)v+\sigma^2 d},\eps\Big)$
   \ENDFOR
\end{algorithmic}
\end{algorithm}

Figure~\ref{fig:LinReg_shift} shows the learning rate schedule $\eta_t^*$ given by Algorithm~\ref{alg:linear}:
\begin{itemize}
    \item {\bf Bursty shifts.} The left subplot corresponds to the setting where $\gamma_t$ follows a jump process. At the beginning of each episode (40 steps each), $\gamma_t$ jumps to a value~$s$ and then becomes zero for the rest of the episode.
    Therefore, the distribution remains the same within an episode but then switches to another distribution in the next episode.
    In this case, we see the learning rate restart at the beginning of each episode with a more aggressive step size
    (capped at $\varepsilon = 0.1$)
    but then decrease within the episode as there is no shift.
    
    \item {\bf Smooth shifts.} The right subplot illustrates the setting where $\gamma_t$ changes continuously as $\gamma_t = 1/t^{\alpha}$
    for a constant value $\alpha$.
    We see that a smaller value of $\alpha$
    (i.e., larger distribution shift) induces a larger learning rate.
\end{itemize}

\begin{figure*}
    \centering
    \hspace{-0.75cm}
    \begin{subfigure}[b]{0.45\textwidth}
    \includegraphics[width=\textwidth]{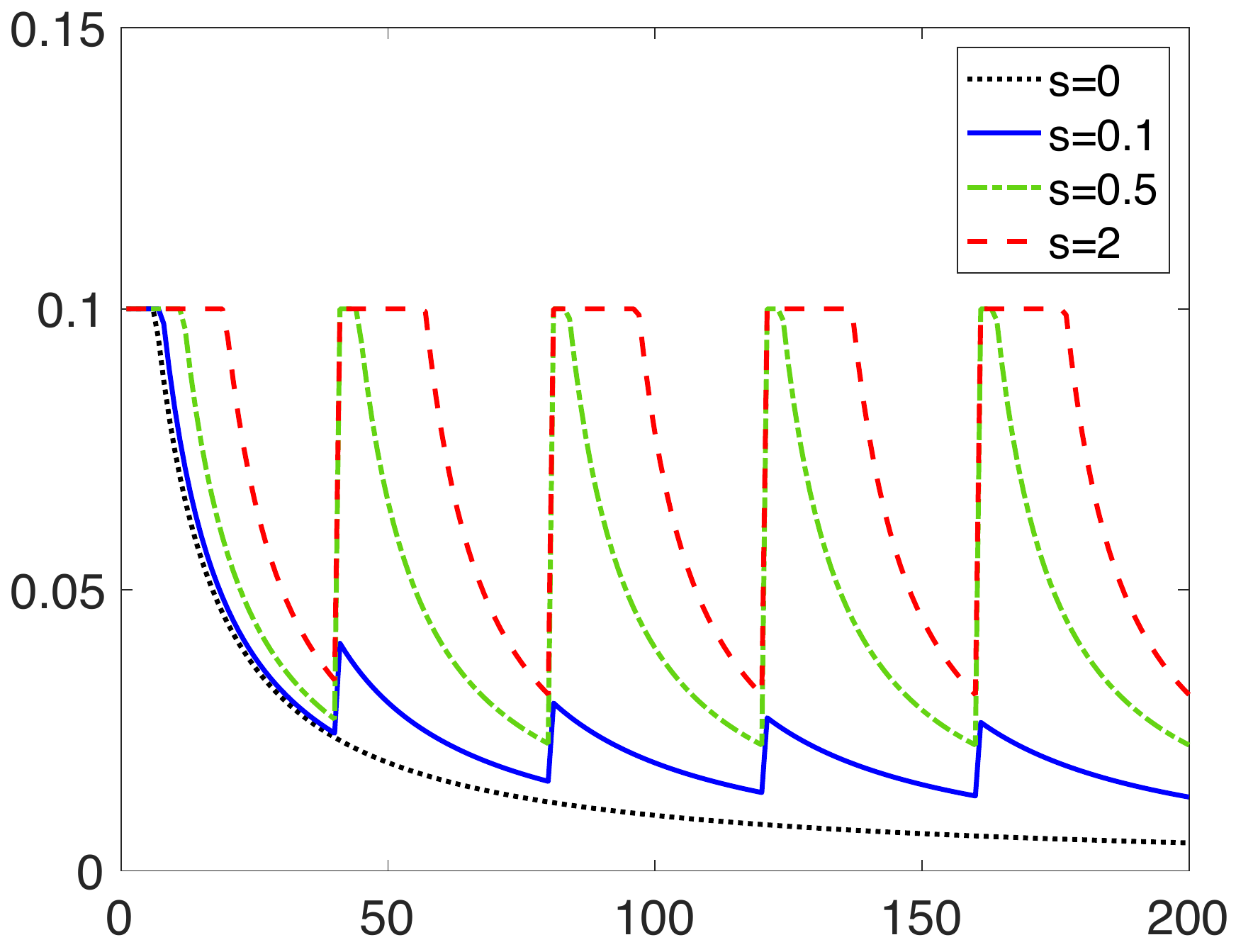}
    \caption{Bursty distribution shifts}\label{fig:LinReg_jump}
    \end{subfigure}
    \hspace{0.15cm}
    \begin{subfigure}[b]{0.45\textwidth}
    %\raisebox{0.14cm}{
        \includegraphics[width=\textwidth]{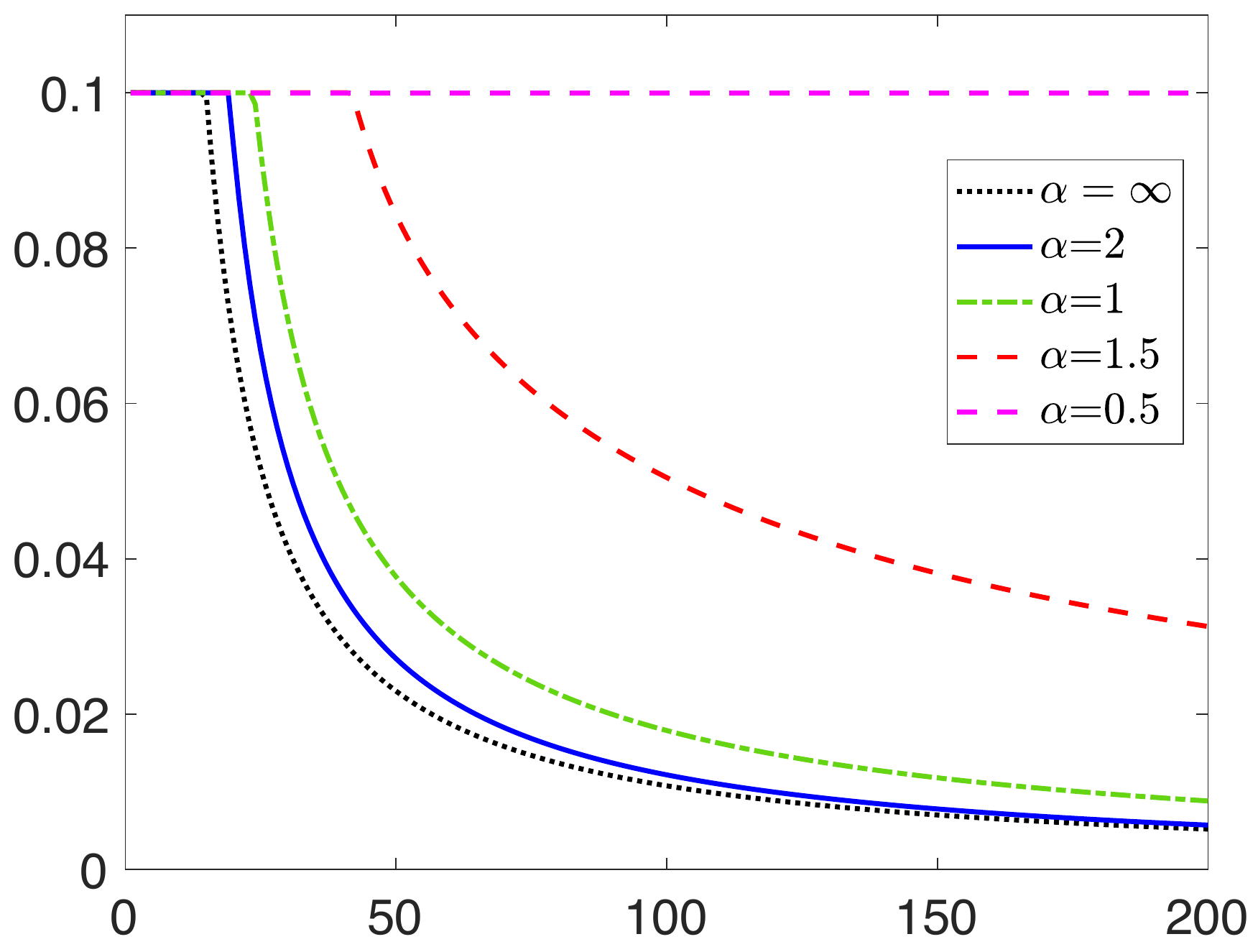}
    %}
    \caption{Smooth distribution shifts}\label{fig:LinReg_smooth}
    \end{subfigure}
    
    \caption{Learning rate schedules $\eta_t^*$ devised in Algorithm~\ref{alg:linear} for online linear regression.
    The batch size is $B_t = 100$ for all $1 \le t \le 200$,
    dimension $d = 100$, max step size $\eps = 0.1$, and $\sigma = 2$.}
    \label{fig:LinReg_shift}
\end{figure*}

\subsection{Case study: No distribution shift}
To  build further insight about the proposed schedule,
we study the behavior of \Cref{alg:linear}
when there is no shift in the data distribution
and the batch size is the same across SGD iterations.
Note that in this case, $Y'(\tau) = 0$ and $\nu(\tau) = B/\eps$.
The behavior of the learning rate schedule $\eta^*_t$ is described in the next lemma.  

\begin{lemma}\label{lem:no-shift-Lin}
Consider the following ODE:
\begin{align}\label{eq:tv-n}
    \tilde{v}'_\tau &=
    \Big({\sf a}\zeta(\tau)^2 -2\zeta(\tau)\Big) \tilde{v}_\tau
     + {\sf b}\zeta(\tau)^2\,\\
     {\sf a} &:=
     \eps\frac{d+1}{B}, \quad {\sf b}:=
     \eps\frac{\sigma^2 d}{B}, \notag
\end{align}
with optimal $\zeta(\tau)$ given by~\eqref{eq:policy}. Define 
\begin{align*}
\tau_*&:= \left[\frac{1}{2-{\sf a}}\log\left((1-{\sf a}) \left(\tilde{v}_0 \frac{2-{\sf a}}{{\sf b}} - 1\right)\right)\right]_+\,,\\
C &= {\sf a}\ln\Big(\frac{1-{\sf a}}{{\sf b}}\Big) +1-{\sf a} -\tau_*\,.
\end{align*}
We then have the following:
\begin{itemize}
    \item If $\tau\le \tau_*$, then
    \begin{align*}
    \tilde{v}_\tau = \left(\tilde{v}_0 - \frac{{\sf b}}{2-{\sf a}}\right) e^{-(2-{\sf a})\tau}+\frac{{\sf b}}{2-{\sf a}}, \quad
    \zeta(\tau)  = 1\,.
    \end{align*}
    \item As $\tau\to \infty$, we have
    \begin{align*}
    \lim_{\tau\to\infty}\frac{\tilde{v}_\tau}{\frac{b}{\tau+C}} = 1\,,\quad
    \lim_{\tau\to\infty} \frac{\zeta(\tau)}{\frac{1}{{\sf a}+C+\tau}} = 1.
    \end{align*}
\end{itemize}
\end{lemma}

Recalling the relation $\eta_t = \eps \zeta(\eps t)$ and using Lemma~\ref{lem:no-shift-Lin}, we have
$\eta^*_t = \eps$ for $t\le t_*:=\lceil \tau_*/\eps\rceil$ and 
\[
    \lim_{t\to\infty} \frac{\eta^*_t}{\frac{\eps}{{\sf a}+ C+\eps t}}  = 1.
\]
In words, $\eta^*_t$ asymptotically has the rate $1/t$. In Figure~\ref{fig:LinReg_LR}, we plot an example of processes $\tilde{v}_\tau$ and the optimal learning rate $\eta^*_t$ for linear regression without any distribution shift.

\begin{figure*}
    \centering
    %\hspace{-0.75cm}
    \begin{subfigure}[b]{0.434\textwidth}
    \includegraphics[width=\textwidth]{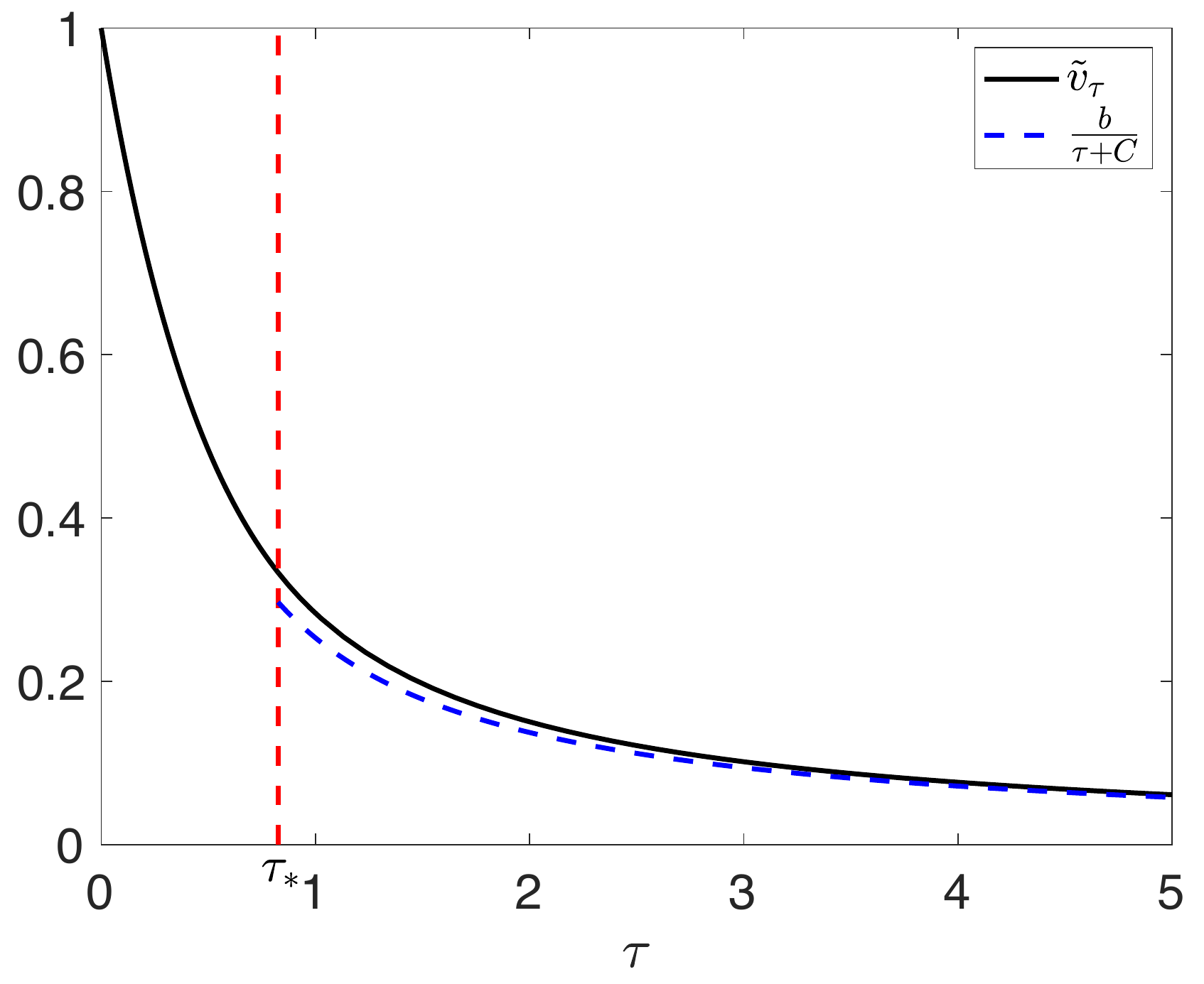}
    %\caption*{(a)}
    \end{subfigure}
    \hspace{0.18cm}
    \begin{subfigure}[b]{0.45\textwidth}
    %\raisebox{0.14cm}{
        \includegraphics[width=\textwidth]{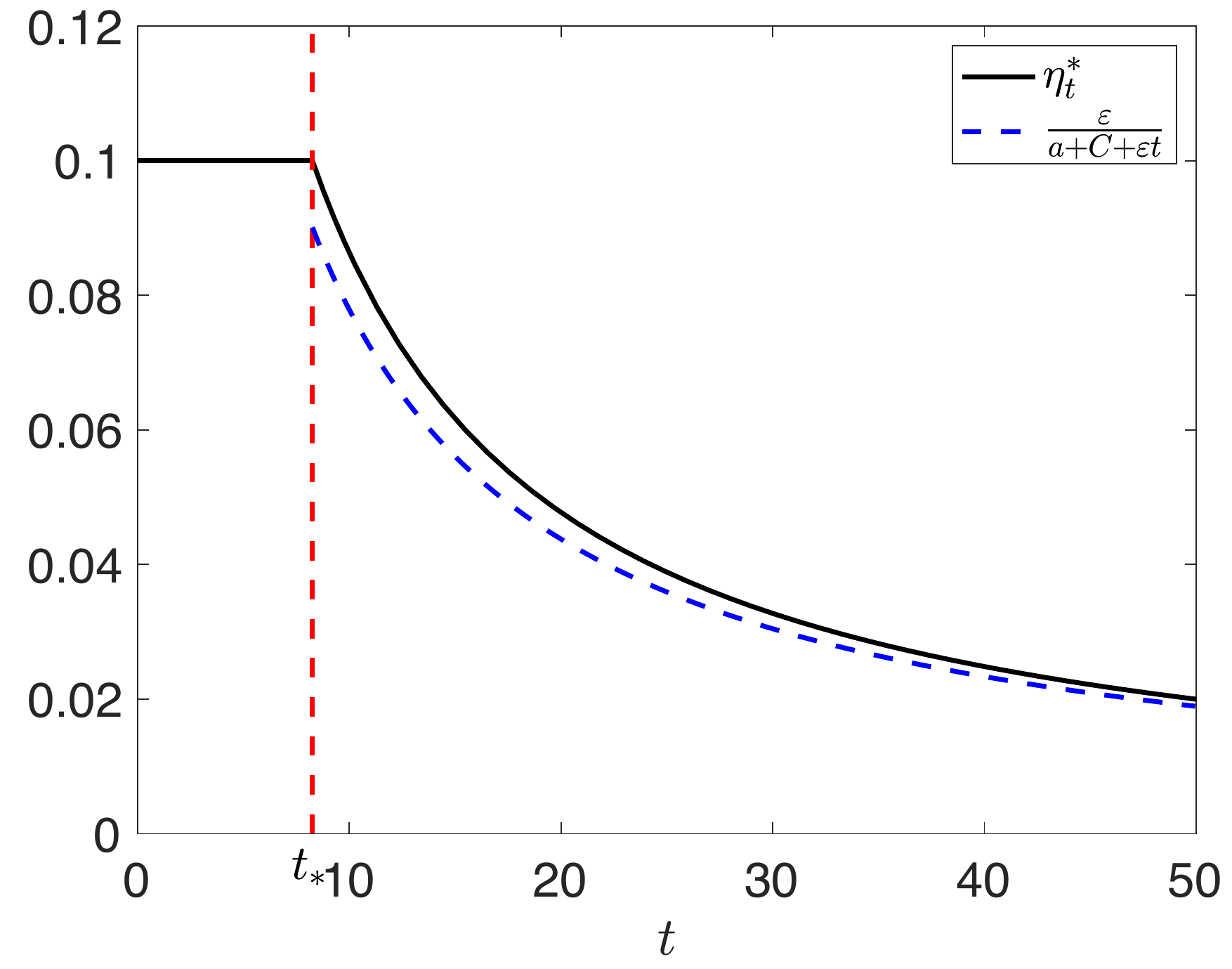}
    %}
    %\caption*{(b)}
    \end{subfigure}
    
    \caption{The process $\tilde{v}_\tau$ defined by ODE~\eqref{eq:tv-n} if there is no distribution shift (left).
    Here we have $\eps = 0.1$, ${\sf a}:=\eps(d+1)/B = 0.1$, ${\sf b}:= \eps \sigma^2 d/B = 0.3$, and initialization $\tilde{v}_0 = 1$.
    Behavior of the learning rate schedule $\eta^*_t$ given by Algorithm~\ref{alg:linear}, which asymptotically has the rate $1/t$ (right).}
    \label{fig:LinReg_LR}
\end{figure*}
\section{General convex loss}
\label{sec:cvx}

\subsection{Upper bound on the total regret}

Here we derive an upper bound on the total regret for general convex loss functions. We use this bound to study the behavior of optimal learning rates (by minimizing the regret upper bound) with respect to distribution shifts. 
We proceed by making the following assumption.
\begin{assumption}\label{ass:main}
Suppose that
\begin{itemize}
    \item[(i)] We have 
    $
    \E_{P_t}[\|\nabla\ell(\theta_t,z_{t,k}) - \nabla \ellbar_t(\theta_t)\|^2]\le \sigma^2,
    $
    for some parameter $\sigma \ge 0$. Since the data points in each batch are sampled i.i.d., this implies that
\[
\E_{P_t}\Big[\Big\|\frac{1}{B_t} \sum_{k=1}^{B_t} \nabla\ell(\theta_t,z_{t,k}) - \nabla \ellbar_t(\theta_t)\Big\|^2\Big]\le \frac{\sigma^2}{B_t}\,.
\]
\item[(ii)] We have $\nabla^2\ellbar_t(\theta)\preceq L I$ for $\theta\in \Theta$, or a weaker $L$-smooth condition
\[
    \|\nabla \ellbar_t(\theta_1) - \nabla \ellbar_t(\theta_2)\|\le L\|\theta_1 - \theta_2\|,
\]
for $\theta_1,\theta_2\in \Theta$.
\item[(iii)] We assume the oracle models $\theta_t^*$ are in $\Theta$ and that the diameter of $\Theta$ is bounded by $D_{\max}$. Alternatively, we assume that $\theta^*_t\in \Theta'$ for all $t$, and $D_{\max} = \max\{\|\theta - \theta'\|: \theta\in\Theta, \theta'\in\Theta'\}$.
\end{itemize}
\end{assumption}
Note that for all steps $t$, $\nabla\ell(\theta_t,z_{t,k})$ is an unbiased estimator of $\nabla\bar{\ell}_t(\theta_t)$ and Assumption $(i)$ bounds its variance. Assumption $(ii)$ is for technical analysis and is satisfied if the loss function has a continuous Hessian.
Assumption $(iii)$ assumes that the oracle models $\theta^*_t$ remain in a bounded set as $t$ grows. Since in practice the SGD is run for a finite number of iterations,
this is not a restricting assumption, e.g.,
$D_{\max}$ can depend on the horizon length $T$.

\begin{thm}\label{thm:convex-reg}
 Suppose the loss function $\ell(\theta,z)$ is convex in $\theta$, and assume that the oracle model~$\theta^*_t$ and the learning rate $\eta_t$ are adapted to the history $\bz_{t-1}$, defined by~\eqref{eq:history}.
 Let $D_t: = \|\theta_t^*-\theta_t\|$ and $a_t: = 2\eta_t - L\eta_t^2 > 0$ for $t\ge 1$. Under Assumption~\ref{ass:main}, and assuming $\eta_t\le \frac{1}{L}$, for all $t\ge 1$, the 
following bound holds on the total regret of SGD:
\begin{align}\label{eq:Reg-UB0}
\E[\Reg(T)] &\le
 \sum_{t=1}^T\E\Bigg[ \left(\frac{D_t^2}{a_t} - \frac{D_{t+1}^2}{a_{t}}\right) + \frac{\sigma^2\eta_t^2}{B_ta_t}
 + \frac{\|\theta^*_t - \theta^*_{t+1}\|^2}{a_t} + \frac{2}{a_t}\<\theta^*_t - \theta^*_{t+1}, \theta_{t+1}-\theta^*_t\>\Bigg]\,.
\end{align}
Here, the expectation is with respect to the randomness in data points observed in the $T$ steps.
\end{thm}
We next discuss how the regret bound~\eqref{eq:Reg-UB0} can be used to derive optimal learning rate schedules.
We would like to derive optimal rates $\eta^*_t$ by minimizing the bound~\eqref{eq:Reg-UB0} in a sequential manner. However, the bound depends on $D_t$ and $\theta_{t+1}$,
which are not observable. Indeed, $\theta_{t+1}$ is defined at step $t+1$ where $\eta_t$ should have already been determined. To address this issue, we use the fact that the projected SGD updates remain in the set $\Theta$ and by invoking Assumption $(iii)$, we have $D_t\le D_{\max}$ and $\|\theta_{t+1}-\theta^*_t\|\le D_{\max}$. Also recall our notation $\gamma_t = \|\theta^*_t - \theta^*_{t+1}\|$ for the distribution shift. Therefore, by rearranging the terms in~\eqref{eq:Reg-UB0} and telescope summing over $1/a_t$, we have
\begin{align}\label{eq:Reg-UB}
    \E[\Reg(T)] &\le
    D_{\max}^2 \E \left[ \frac{1}{a_1} +
    \sum_{t=2}^T \left(\frac{1}{a_t} - \frac{1}{a_{t-1}}\right)_+ \right]
    +\sum_{t=1}^T\E\left[\frac{1}{a_t} \left(\frac{\sigma^2\eta_t^2}{B_t} + \gamma_t^2 + 2D_{\max}\gamma_t\right)\right],
\end{align}
where $x_+ = \max(x,0)$ indicates the positive part of $x$.

We next discuss the choice of learning rates that minimizes the upper bound \eqref{eq:Reg-UB} in a sequential manner. Conditioned on $\bz_{[t-1]}$, the optimal $\eta_t$ is given by
\begin{align}\label{eq:etat*0}
    \eta_t^*:= \underset{0\le \eta\le \frac{1}{L}}{\text{argmin}} \bigg\{&D_{\max}^2
    \left(\frac{1}{2\eta-L\eta^2} - \frac{1}{2\eta_{t-1}-L\eta_{t-1}^2}\right)_+
    +\frac{\sigma^2}{B_t}\cdot\frac{\eta^2}{2\eta-L\eta^2}+ \frac{\gamma_t^2+2D_{\max}\gamma_t}{2\eta-L\eta^2}\bigg\}\,.
\end{align}
Our next proposition characterizes $\eta_t^*$.
\begin{propo}[Learning rate schedule]
\label{propo:optimal-eta}
Define the thresholds $\tau_{1,t}$ and $\tau_{2,t}$ as follows:
\begin{align}
    \tau_{1,t}&:= \tfrac{B_t}{2\sigma^2}\left(\sqrt{b_{1,t}^2L^2+\tfrac{4\sigma^2}{B_t} b_{1,t}} - b_{1,t}L\right),  \label{eq:tau1}\\
    \tau_{2,t}&:= \tfrac{B_t}{2\sigma^2}\left(\sqrt{b_{2,t}^2L^2+\tfrac{4\sigma^2}{B_t} b_{2,t}} - b_{2,t}L\right), \label{eq:tau2}\\
    b_{1,t} &:= \gamma_t^2+2D_{\max}\gamma_t,\quad  b_{2,t} := (\gamma_t+D_{\max})^2\,.\nonumber 
\end{align}
The optimal learning rate $\eta_t^*$ defined by~\eqref{eq:etat*0} is given by:
\begin{align}
    \eta^*_t = \begin{cases}
    \tau_{1,t} & \text{if } \eta_{t-1}^* \le \tau_{1,t},\\
    \eta_{t-1}^* & \text{if } \tau_{1,t} \le \eta_{t-1}^* \le \tau_{2,t}\\
    \tau_{2,t} & \text{if } \eta_{t-1}^* \ge \tau_{2,t}\,.
    \end{cases}
\end{align}
\end{propo}

\begin{remark}
The proposed learning rate in \eqref{eq:etat*0} depends on $\sigma$, $L$ and shifts $\gamma_t$.
Having access to the loss function $\ell(\theta,z)$, the learner can use sample estimates for $\sigma$, $L$. Also note that we can use any upper bound on $\gamma_t$ in the bound~\eqref{eq:Reg-UB} and obtain a similar schedule.
Of course, if the upper bound is crude, it results in a conservative learning rate schedule. In settings where an upper bound on the shifts $\gamma_t$ is not available, we estimate $\gamma_t$ using an exponential moving average of the drifts in the consecutive estimated models $\theta_t$, namely $\hat{\gamma}_t = \beta \hat{\gamma}_{t-1} + (1-\beta)\|\theta_t- \theta_{t-1}\|$, with a factor $\beta\in(0,1)$.
%$\hat{\gamma}_t  = \frac{1}{W}\sum_{k=t-W+1}^{t}\|\theta_k  - \theta_{k-1}\|$, for a window size $W$.
\end{remark}

\begin{remark}\label{rmk:opt-sch}
The values $b_{1,t}$ and $b_{2,t}$ in~\eqref{eq:tau1} and \eqref{eq:tau2} are increasing in the distribution shift $\gamma_t$ and it is easy to see that the thresholds $\tau_{1,t},\tau_{2,t}$ are also increasing in $\gamma_t$. As a result for every value of $\eta_{t-1}$, higher distribution shift $\gamma_t$ increases the optimal learning rate $\eta^*_t$.
\end{remark}

Note that Theorem~\ref{thm:convex-reg} and Remark~\ref{rmk:opt-sch} are optimized with respect to the upper envelope of the optimal regret. We also prove a corresponding lower envelope result for SGD. 

\subsection{Lower bound on the total regret}

 The learning rate schedule in \ref{propo:optimal-eta} is optimized with respect to the upper bound derived for the cumulative dynamic regret. We next prove a corresponding lower bound result for SGD, which matches the upper bound and only differs by constants.
 Thus, our analysis of the optimal learning rate schedules for SGD is tight up to constants.

Before we begin, we make an additional assumption.

\begin{assumption}\label{as:subexp0}
We assume that the loss function $\ell(\theta,z)$ is $\mu$-strongly convex in $\theta$, for some $\mu>0$, i.e., $\ell(\theta) - \frac{\mu}{2}\|\theta\|^2$ is convex in $\theta$. 
\end{assumption}

\begin{thm}\label{thm:le-convex-reg}
 Suppose the oracle model $\theta^*_t$ and the learning rate $\eta_t$ are adapted to the history $\bz_{t-1}$, defined by~\eqref{eq:history}. Let $D_t := \|\theta^*_t - \theta_t\|$, $\gamma_t := \|\theta^*_t - \theta^*_{t+1}\|$, and $a'_t: = 2(\eta_t+\frac{L}{\mu}\eta_t - \eta_t^2L)$. Under Assumptions~\ref{ass:main} and~\ref{as:subexp0}, and assuming $\eta_t\le \frac{1}{\mu}$, for all $t\ge 1$, we have the following bound on the total regret of the batch SGD:
%$a_t': = \frac{c}{D_{max}^\alpha}\left(2\eta_t-\frac{\eta_t^2}{c}\right)$
\begin{align}\label{eq:Reg-LB}
\E[\Reg(T)]&\ge
 \sum_{t=1}^T\E\bigg[ \left(\frac{D_t^2}{a'_t} - \frac{D_{t+1}^2}{a'_{t}}\right) + \frac{\sigma^2\eta_t^2}{B_ta'_t} +\frac{\|\theta^*_t-\theta^*_{t+1}\|^2}{a'_t}  +\frac{2}{a_t'}\<\theta^*_t-\theta^*_{t+1},\theta_{t+1}-\theta^*_t\>\bigg],
\end{align}
where the expectation is with respect to the randomness in data points observed in the $T$ steps.
\end{thm}

\noindent
Note that Equations~\eqref{eq:Reg-LB} and~\eqref{eq:Reg-UB0} have the same form and thus give upper and lower ``envelopes'' for the cumulative expected dynamic regret under these assumptions.

One possible interpretation of terms in bounds~\eqref{eq:Reg-UB0} and~\eqref{eq:Reg-LB} is a predator-prey setting, as follows. The predator is $\theta_t$ and the prey is $\theta^*_t$.
    The first term in \eqref{eq:Reg-UB0} can be rearranged as 
    \[
    \sum_{t=1}^T \left(\frac{D_t^2}{a_t} - \frac{D_{t+1}^2}{a_t}\right)
    =
    \frac{D_1^2}{a_1}- \frac{D_{T+1}^2}{a_T} +
    \sum_{t=2}^T D_t^2\left(\frac{1}{a_t} -\frac{1}{a_{t-1}}\right)\,.
    \]
    Recall that $D_t = \|\theta_t-\theta^*_t\|$ is the distance between the prey and the predator. If the predator moves closer to the prey, it reduces its regret. 
    The other terms in \eqref{eq:Reg-UB0} involve  $\theta^*_t-\theta^*_{t+1}$ and reflect the movement of $\theta^*_t$ (the prey).
    If the prey moves further, it is harder to follow and the regret increases.
\section{Non-convex loss}
\label{sec:nonconvex}
When the loss function $\ell$ is non-convex, SGD like any other first order method can get trapped in a local minimum or a saddle point of the landscape. When there is no distribution shift,
there is a line of work showing that SGD can efficiently escape saddle points if the step size is large enough~\citep{lee2016gradient,jin2017escape}.
This superiority of SGD in non-convex settings is often attributed to  the stochasticity of the gradients,
which significantly accelerates the escape from saddle points.

In non-convex settings one cannot control convergence to a global minimum without making further structural assumption on the optimization landscape and the initialization of SGD. In view of that, we propose to consider the following notion of regret based on the cumulative gradient norm of the SGD trajectory:
\begin{align}
    \Reg(T) := \sum_{t=1}^T \|\nabla\ellbar_t(\theta_t)\|^2\,.
\end{align}
In words, the regret is defined with respect to the norm of gradient at the sequence of estimated models. This notion does not differentiate between local or global minima.

Further, due to the complex landscapes of non-convex loss,
we work with a more holistic measure of distribution shift, namely
\begin{align}\label{eq:gamma-Nconv}
    \gamma_t: = \sup_{\theta \in \mathbb{R}^p} |\ellbar_t(\theta) - \ellbar_{t+1}(\theta)|\,.
\end{align}
Recall that $\ellbar_t = \E_{P_t}[\ell(\theta,z_{t,k})]$ and obviously if there is no shift at step $t$, i.e., $P_t = P_{t+1}$ then $\gamma_t=0$. In contrast, in the convex setting, we measure the distribution shift only in terms of the difference between the global minimizers of $\ellbar_t$ and $\ellbar_{t+1}$, cf. Definition~\ref{def:shift}.

% We establish a regret bound for the non-convex setting under Assumption~\ref{ass:main} $(i),(ii)$. 
% Note that by part $(ii)$ of this assumption, $\|\nabla\ellbar_t(\theta)\|$ is continuous and so achieves its maximum over the bounded set $\Theta$. Define 
% \begin{align}\label{eq:M}
% M: = \sup_{\theta\in\Theta}\|\nabla \ellbar_t(\theta)\|.
% \end{align}
We can now state our regret bound in the non-convex setting.
\begin{thm}\label{thm:non-Nconvex-reg}
Suppose the learning rates $\eta_t$ are adapted to the history $\bz_{t-1}$, defined by~\eqref{eq:history}. Let $\gamma_t$ be defined as~\eqref{eq:gamma-Nconv}, and define $a_t: = 2\eta_t - L\eta_t^2$, for $t\ge 1$. Under Assumption~\ref{ass:main} $(i),(ii)$, and assuming $\eta_t\le \frac{1}{L}$, for all $t\ge 1$, we have the 
following bound on the total regret of batch SGD:
\begin{align}
\label{eq:Reg-UB-Nconv}
\E[\Reg(T)]
    &\le 
    \E\left[
      \frac{2\ellbar_1(\theta_1)}{a_1^2}
    + \sum_{t=2}^T 2\ellbar_t(\theta_t) \left(\frac{1}{a_t} - \frac{1}{a_{t-1}^2}\right)\right]
    +\sum_{t=1}^T\E\left[\frac{1}{a_t} \cdot \left(\frac{L\sigma^2\eta_t^2}{B_t} + 2\gamma_t\right)\right]\,.
\end{align}
%Here, the expectation is with respect to the randomness in data points observed in the $T$ steps.
\end{thm}
The theorem above has a very similar format to the bound derived in Theorem~\ref{thm:convex-reg}.
By minimizing the regret of the upper bound
\eqref{eq:Reg-UB-Nconv} in sequential manner
conditioned on $\bz_{[t-1]}$,
the optimal learning rate is given by
\begin{align}\label{eq:etat*-Nconv}
    \eta_t^*:= \argmin_{0\le \eta\le \frac{1}{L}} 
    \frac{2\ellbar_t(\theta_t)+2\gamma_t}{2\eta-L\eta^2} 
    +\frac{L\sigma^2}{B_t}\cdot\frac{\eta^2}{2\eta-L\eta^2}\,.
\end{align}
The optimal $\eta^*_t$ admits a closed form solution given below:
\[
\eta^*_t = \tfrac{B_t}{L\sigma^2}\left(\sqrt{b_t^2+2\tfrac{\sigma^2}{B_t} b_t} - b_t\right)\,, \; b_t = L(\gamma_t+\ellbar_t(\theta_t))\,.
\]
The above characterization is derived by noticing that the function in~\eqref{eq:etat*-Nconv} is convex in $\eta$, for $\eta\in (0,1/L]$ and the stationary point of the function $\eta^*$ satisfies the boundary condition $0\le \eta^*\le 1/L$.

%
% \begin{propo}\label{propo:optimal-eta-Nconv}
% Define the thresholds $\tau_1$ and $\tau_2$ as follows:
% \begin{align*}
%     \tau_{1,t}&:= \frac{B_t}{\sigma^2}\left(\sqrt{\gamma_t^2+\frac{2\sigma^2}{LB_t} \gamma_t} - \gamma_t\right), \\
%     \tau_{2,t}&:= \frac{B_t}{\sigma^2}\left(\sqrt{(\gamma_t+M)^2+\frac{2\sigma^2}{LB_t} (\gamma_t+M)} - \gamma_t-M\right), 
% \end{align*}
% The optimal learning rate $\eta_t^*$ defined by~\eqref{eq:etat*} is given by:
% \begin{align}
%     \eta^*_t = \begin{cases}
%     \tau_{1,t} & \text{if } \eta_{t-1}\le \tau_{1,t},\\
%     \eta_{t-1} & \text{if } \tau_1\le \eta_{t-1}\le \tau_{2,t}\\
%     \tau_{2,t} & \text{if } \eta_{t-1}\ge \tau_{2,t}\,.
%     \end{cases}
% \end{align}
% \end{propo}
% \begin{remark}
% It is easy to verify that the thresholds $\tau_{1,t},\tau_{2,t}$ are increasing in $\gamma_t$. Therefore, for every value of $\eta_{t-1}$  the optimal learning rate $\eta^*_t$ is increasing in distribution shift $\gamma_t$.
% \end{remark}
It is easy to see that the learning rate $\eta^*_t$ is increasing in the distribution shift $\gamma_t$. To implement this learning rate, we estimate $\ellbar_t(\theta_t)$ by $\ell^{B_t}(\theta_t)$, its sample average over the batch at time $t$.
The proofs are deferred to the supplementary materials
due to the space constraint.
\section{Experiments}
\label{sec:experiments}

We use TensorFlow \citep{abadi2016tensorflow} and Keras~\citep{chollet2015keras} for the following experiments.\footnote{The source code is  available at \url{https://github.com/fahrbach/learning-rate-schedules}.}
In \Cref{subsec:high_dimensional_regression} we study high-dimensional regression, and
in \Cref{app:cyto} we explore an application of neural networks to flow cytometry.

\subsection{High-dimensional regression}
\label{subsec:high_dimensional_regression}

We use the learning rate schedules
in \Cref{alg:linear} and \Cref{propo:optimal-eta}
for linear and logistic regression, respectively.
We consider paths $\{\theta_{t}^*\}_{t = 1}^{T}$
such that for $\theta_{t}^* \in \mathbb{R}^d, i \in [d]$,
\begin{equation}
\label{eqn:theta_path_def}
    \theta_{t}^*(i) = \begin{cases}
        r_{a,b}(t)^3 \cos(\lceil i/2 \rceil 2k\pi \alpha(t)) & \text{if $i$ odd}, \\
        r_{a,b}(t)^3 \sin(\lceil i/2 \rceil 2k\pi \alpha(t)) & \text{if $i$ even}, \\
    \end{cases}
\end{equation}
where $r_{a,b}(t) = \texttt{linspace}(a,b,T)$ controls the radius,
$\alpha(t) = \texttt{linspace}(0,1,T)$, and $k$ is the base frequency.
These paths have linearly independent components due to their trigonometric
frequencies and phases
(useful for high dimensions),
and move at non-monotonic speeds if $a \ne b$.

\begin{figure*}
    \centering
    \begin{subfigure}[b]{0.31\textwidth}
    \includegraphics[width=\textwidth]{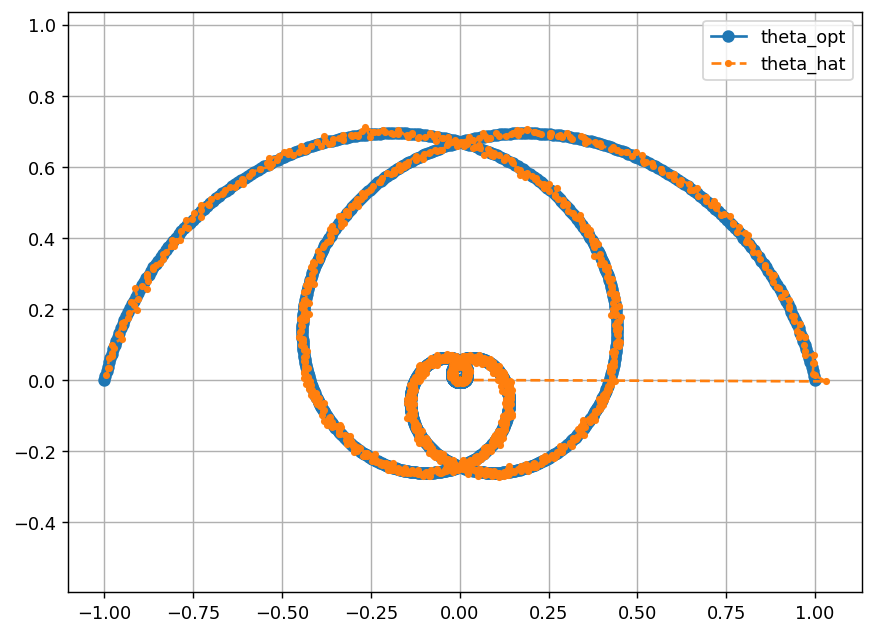}
    \end{subfigure}
    \hfill
    \begin{subfigure}[b]{0.31\textwidth}
    \includegraphics[width=\textwidth]{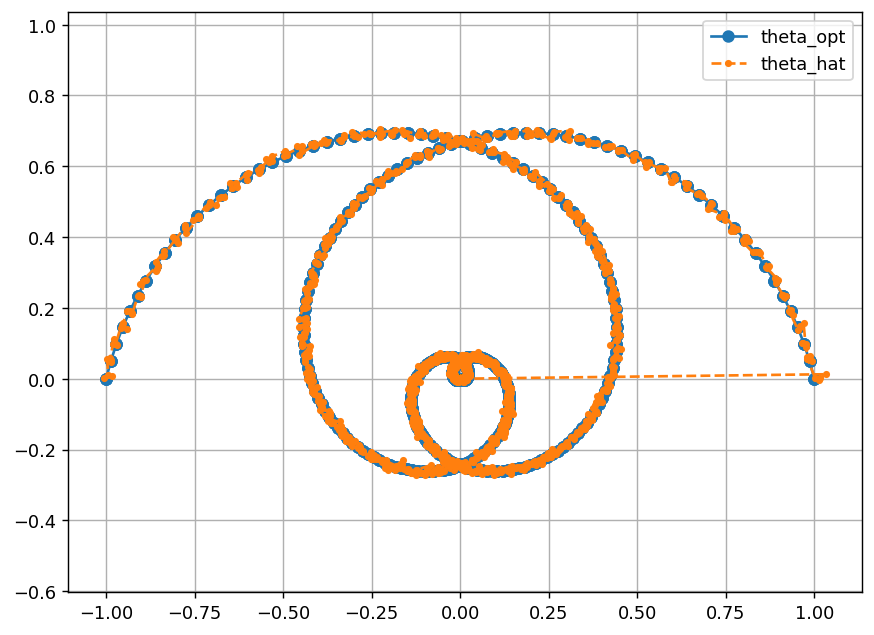}
    \end{subfigure}
    \hfill
    \begin{subfigure}[b]{0.31\textwidth}
    \includegraphics[width=\textwidth]{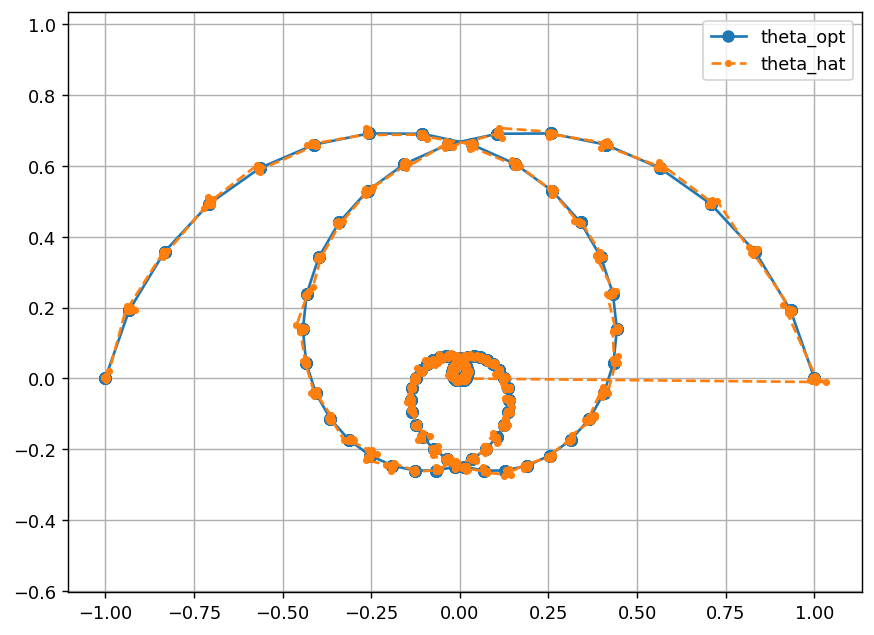}
    \end{subfigure}
    \\
    \begin{subfigure}[b]{0.31\textwidth}
    \includegraphics[width=\textwidth]{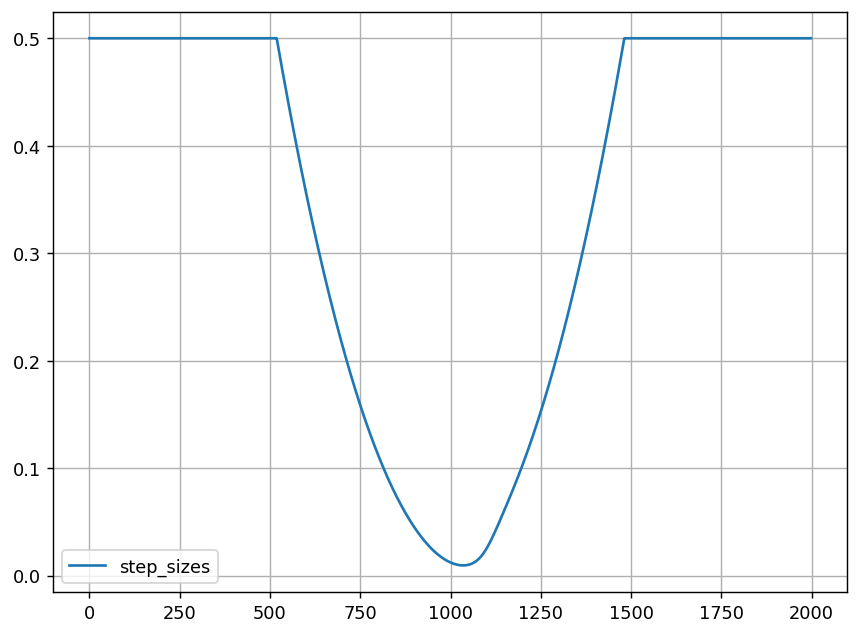}
    \end{subfigure}
    \hfill
    \begin{subfigure}[b]{0.31\textwidth}
    \includegraphics[width=\textwidth]{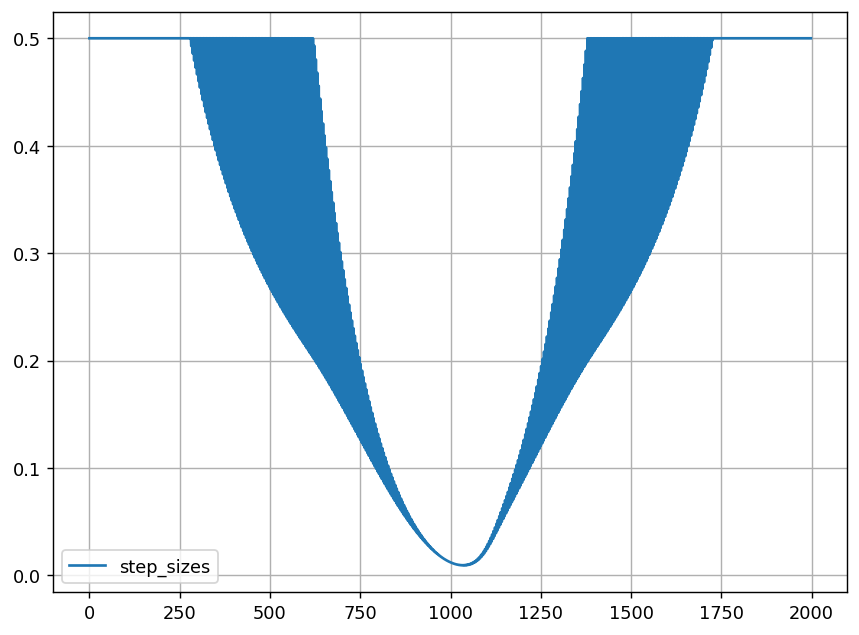}
    \end{subfigure}
    \hfill
    \begin{subfigure}[b]{0.31\textwidth}
    \includegraphics[width=\textwidth]{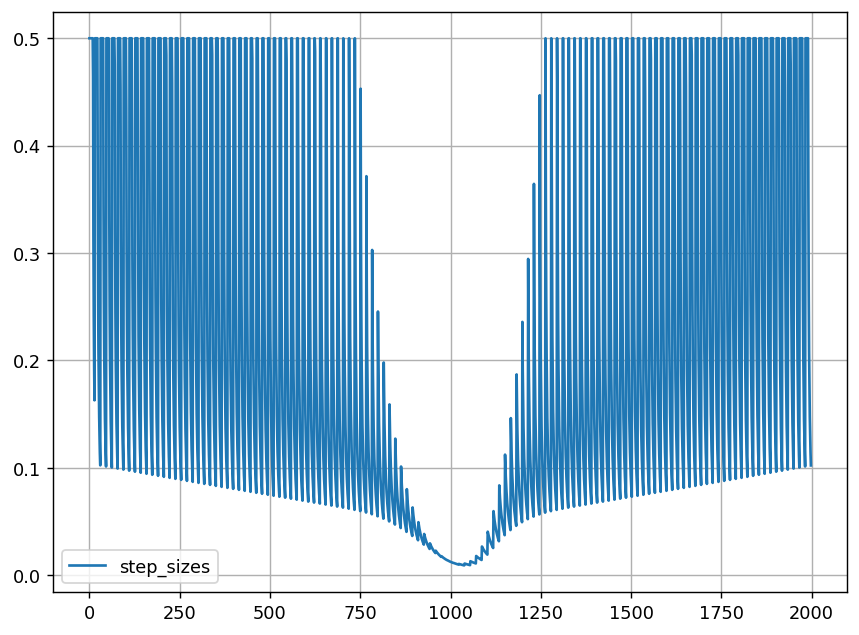}
    \end{subfigure}
    
    \caption{SGD trajectories of \Cref{alg:linear} (top);
    and oscillating learning rates $\eta_t$ as we
    discretize the path defined by $\theta_{t}^*$ where $\eta_{\max} = 0.5$ (bottom).}
    \label{fig:linear_regression_continuous_vs_discrete}
\end{figure*}

\subsubsection{Linear regression}
We start by investigating \Cref{alg:linear} for online least squares.
Setting $\theta_0 = 0$,
at each step $t$ we generate $X \in \mathbb{R}^{B_t \times d}$
for $x_{ij} \sim \normal(0, 1)$
and get back the response
$y = X \theta_{t}^* + \varepsilon$ for $\varepsilon_{i} \sim \normal(0, 0.1)$.

%\paragraph{Continuous vs discrete shifts.}
Consider the 2-dimensional trajectory in~\Cref{fig:linear_regression_continuous_vs_discrete}
defined by $r_{1,-1}(t)$, $k=4$, and $B_t=256$.
For $T=2000$, the path starts at 
$\theta_{1}^* = (1,0)$,
spirals into the origin,
and returns to $\theta_{T}^* = (-1,0)$.
To study the effect of \emph{continuous vs discrete distribution shifts},
we downsample the points by $\ell \in \{1,4,16\}$
to get the discretized paths
\[\hat{\theta}_{t}^{*} = \theta_{\lceil t / \ell \rceil \ell}^*,\]
for $t \in [T]$.
As $\ell$ increases
(i.e., from left to right in \Cref{fig:linear_regression_continuous_vs_discrete}),
the learning rate $\eta_t$ of \Cref{alg:linear} starts to oscillate---decreasing
when $\theta_t$ is near $\theta_t^*$
and returning to $\eta_{\max}=0.5$ when $\theta_{t}^*$ shifts.

Next, we increase the dimension $d$ and
plot the cumulative regret of~\Cref{alg:linear}
in \Cref{fig:regression-high-dimensions}.
We use the same $\ell=8$ discretized paths
and set $\eta_{\max} = 1/\sqrt{d}$.
Note that for all values of $d$,
the total regret increases, levels off,
and then increases again.
This corresponds to $\theta_{t}^*$ spiraling into the origin,
spending time there, and exiting.
The initial spike in regret is due to finding the $\theta_{t}^*$ path,
i.e., the first few steps when
$\theta_t$ moves from the origin to $\theta_{t}^*$.

\subsubsection{Logistic regression}
%\paragraph{Logistic regression.}

% Layout 1: Two rows, large figures
\begin{figure*}
    \centering
    \hspace{-0.75cm}
    \begin{subfigure}[b]{0.45\textwidth}
    \includegraphics[width=\textwidth]{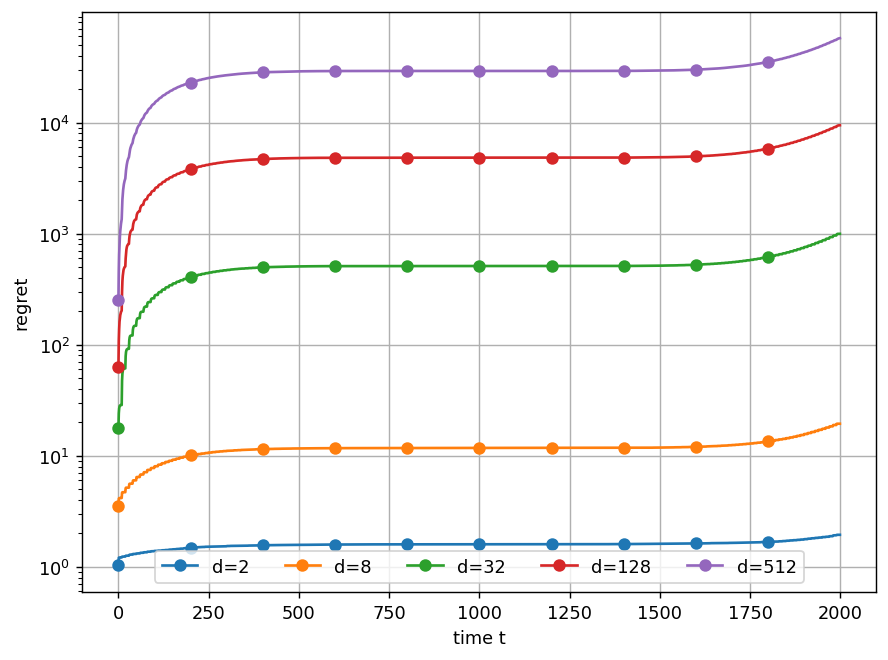}
    %\caption*{(a)}
    \end{subfigure}
    \hspace{0.15cm}
    \begin{subfigure}[b]{0.445\textwidth}
    \raisebox{0.14cm}{
        \includegraphics[width=\textwidth]{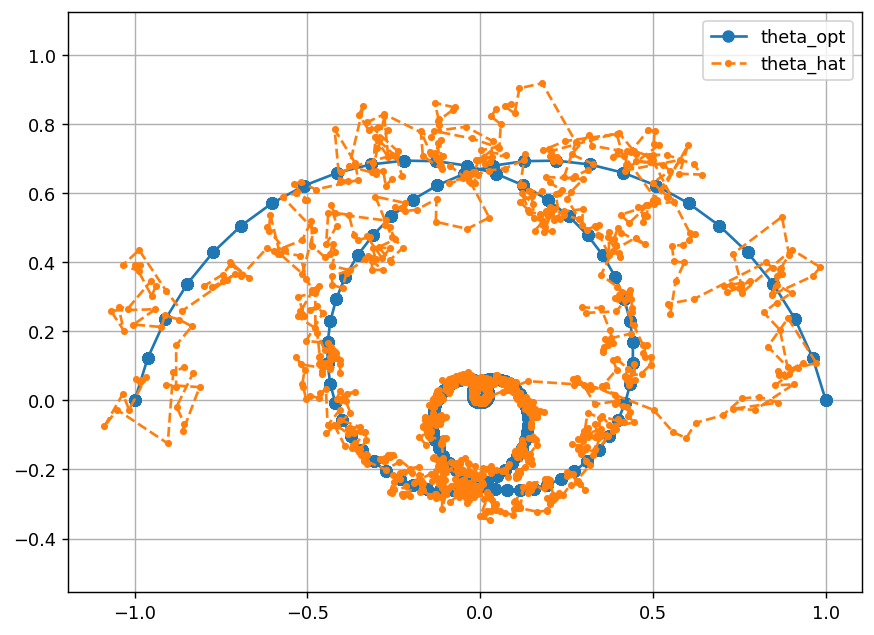}
    }
    %\caption*{(b)}
    \end{subfigure}
    \\
    \hspace{-0.75cm}
    \begin{subfigure}[b]{0.45\textwidth}
    \includegraphics[width=\textwidth]{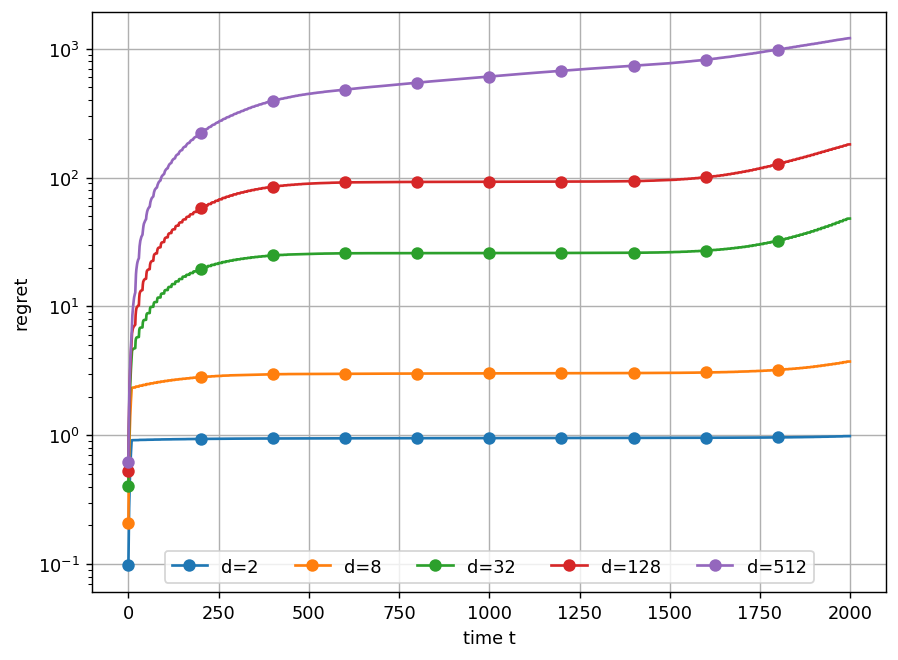}
    %\caption*{(c)}
    \end{subfigure}
    \hspace{0.15cm}
    \begin{subfigure}[b]{0.445\textwidth}
    \raisebox{0.14cm}{
        \includegraphics[width=\textwidth]{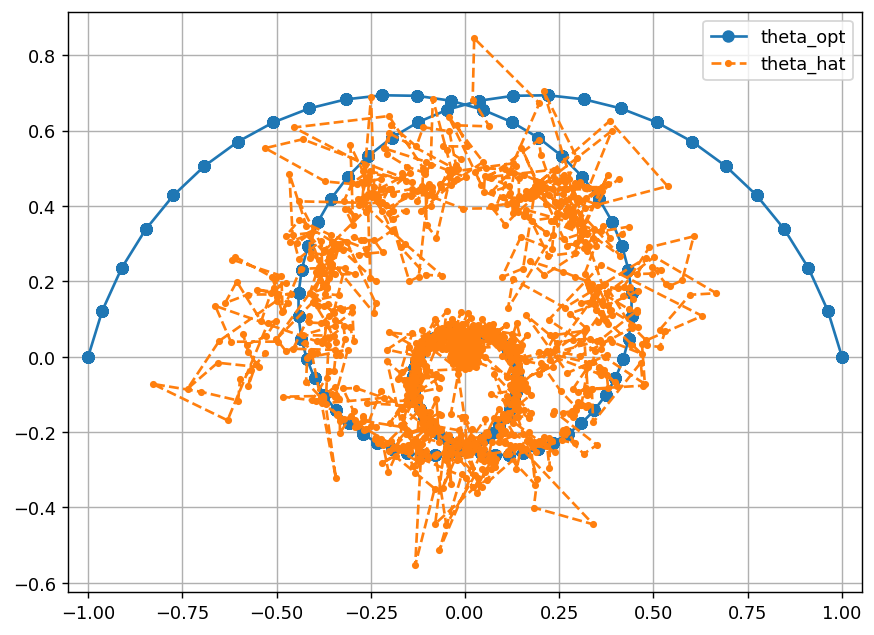}
    }
    %\caption*{(d)}
    \end{subfigure}
    
    \caption{Cumulative regret of \Cref{alg:linear} with $\eta_{\max} = 1/\sqrt{d}$
    for increasing dimensions $d$ (top-left);
    and the first and second coordinates of the SGD
    for $d=128$ and batch size $B_t = 256$ (top-right).
    Cumulative regret of \Cref{propo:optimal-eta} for $d$-dimensional logistic regression (bottom-left);
    and the first and second coordinates of the SGD
    for $d=128$ and batch size $B_t = 256$ (bottom-right).}
    \label{fig:regression-high-dimensions}
\end{figure*}

We also empirically study the learning rate schedule in~\Cref{propo:optimal-eta}
for $d$-dimensional logistic regression with binary cross entropy loss.
Similar to the linear regression experiments,
at each step $t$ we generate the covariates
$X \in \mathbb{R}^{B_t \times d}$,
but now we get back
$y = \text{sigmoid}(X \theta_{t}^* + \varepsilon)$.
We note that the learning rate schedule in \Cref{propo:optimal-eta}
is largely parameter-free for generalized linear models.
For example, setting $\sigma^{2}=d/4$ and $L=1/4$
minimizes the upper bound on the regret in \eqref{eq:Reg-UB}
for logistic regression with log loss,
so the only hyperparameter we set is $D_{\max} = d$.
\subsection{Flow cytometry}
\label{app:cyto}

Next we explore a medical application called \emph{flow cyotometry},
which uses neural networks and online stochastic optimization to
classify cells as they arrive in a stream from a shifting data distribution.
The features this model receives as input are measurements based on the
RNA expressions of each cell
(see, e.g.,~\citet{LMC19,cyt3,cyt2} and the references therein for details).
This induces a learning problem with a
non-convex loss landscape that changes with time,
where we do not have a tight characterization for an optimal learning rate schedule.

\subsubsection{Background}

We start with background on flow cytometry
to give more context for this application.
A sample of cells from a tissue is prepared
and a small number of selected RNA sequences in the cells
are bound to different fluorescent markers.
A laser then illuminates the incoming stream of cells,
which can now be separated based on the
intensity of the signals from different fluorescent markers.
Using fluorescent markers, however, comes at a cost
as they can interfere with normal cellular functioning.
In contrast, marker-free systems that use large convolutional neural nets
are often more accurate, but can be slower to adapt to distribution shifts. See~\citet{LMC19} for further details. 

We study a two-step system that does initial classification
with an inexpensive ``student'' neural network and only relies on a
small number of fluorescent markers.
This is followed by additional analysis
using a large pretrained convolutional neural network (CNN)
with near real-time feedback.
As a simplification, we assume the expensive CNN is a ``teacher'' model
whose predictions are ground truth labels.
We can achieve real-time feedback for the initial classifier
that first sees the cells by replicating the teacher across
servers to increase its inference throughput.
The goal is to optimize the (inexpensive) classifier online
and minimize its loss,
i.e., the number of misclassified cells.
%\footnote{If the feedback is not real-time, then training the initial simpler neural net with larger batch size can capture the effect of delay in feedback due to the slower and larger neural net.}
%The goal is to train the smaller neural net online to optimize for loss---the number of misclassified cells.

The distribution of the arriving cells can change based on the
sample preparation and tissue characteristics.
For example, for pancreatic tissue, if we stream the cells starting from anterior to posterior,
the initial mixture of cells consists of more non-secreting cells
but later will have a higher proportion of secreting cells.
Thus, as a simplification, it is worth exploring the effect of different learning rate schedules for a simple online neural network
that classifies the input stream of cells into different cell types
based on a small number of RNA expression markers in each cell.
We use the pancreatic RNA expression data in
\cite{Bastdas-Ponce,Bergen20}.\footnote{This data is available at \url{https://scvelo.readthedocs.io/scvelo.datasets.pancreas/}.}
%and a single layer neural net with ELU activation to simulate this set-up. See the details below.

Specifically, we use the expression levels of ten RNA molecules
(corresponding to genes Pyy, Meg3, Malat1, Gcg, Gnas, Actb, Ghrl, Rsp3, Ins2 and Hspa8)
for the $4000$ murine pancreatic cells in the
\href{https://scvelo.readthedocs.io/scvelo.datasets.pancreas/}{\texttt{scVelo}} repository.
The expression levels of these genes determines the cell types completely.
We slightly perturb the expression levels to generate a stream of cells,
and within this stream we vary the distribution of secreting cells
(i.e., alpha, beta, and delta)
and non-secreting cells (i.e., ductal),
starting from non-secreting cells dominating the distribution
and ending with secreting cells dominating the distribution.
\Cref{fig:flow_cytometry} (left) is a two-dimensional embedding of these ten
signals labeled by their cell-type.
In practice, any stream of cells undergoes a similar distribution shift
depending on how the samples are prepared.

\begin{figure*}
    \centering

    \hspace{-0.3cm}
    \raisebox{0.3cm}{
        \includegraphics[width=0.45\textwidth]{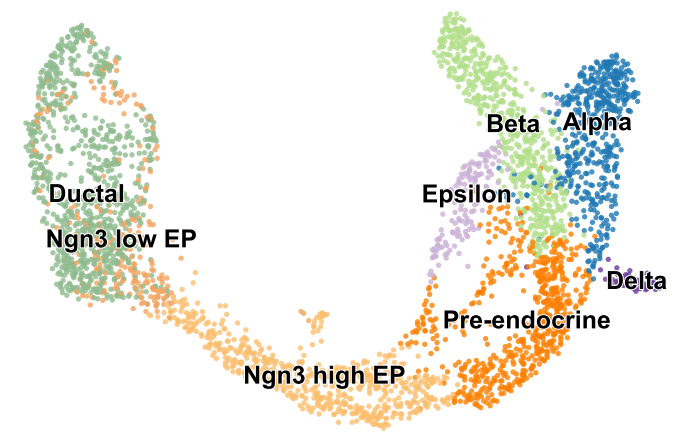}
    }
    %\caption{}\label{fig:cytometry_visualization}
    \hspace{0.2cm}
    \includegraphics[width=0.48\textwidth]{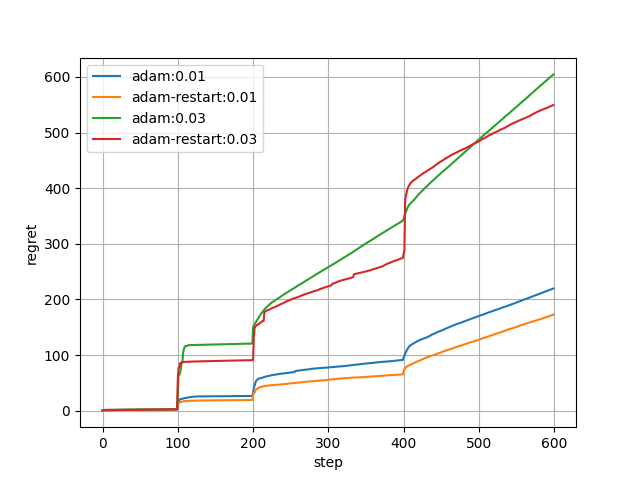}
    \caption{
    Visualization of the 10-dimensional cytometry data and their ground truth labels (left).
    Cumulative regret of online models using different initial learning rates and
    optional Adam restarts at the beginning of each distribution shift (right).
    }\label{fig:flow_cytometry}
    \vspace{-0.2cm}
\end{figure*}

\subsubsection{Experimental setup: Model and cytometry simulation}
The following is a description of our simulation setup:

\begin{itemize}
    \item \emph{Training data and distribution shift:}
    Each training example is a $10$-dimensional vector $x \in \mathbb{R}^{10}$
    drawn from a mixture distribution of 4000 murine pancreatic cells
    and updated by randomly perturbing each of its RNA expressions
    by a factor $U \sim [0.9, 1.1]$ drawn i.i.d.
    The label $y$ is the cell type: ductal, alpha, beta, delta.
    We consider a shift between four different mixture distributions:
    \begin{enumerate}
        \item $P_1(y) = (0.0, 0.0, 0.0, 1.0)$ for $100$ steps
        \item $P_2(y) = (0.0, 0.0, 0.1, 0.9)$ for $100$ steps
        \item $P_3(y) = (0.1, 0.0, 0.2, 0.7)$ for $200$ steps
        \item $P_4(y) = (0.3, 0.5, 0.1, 0.1)$ for $200$ steps
    \end{enumerate}
    The first distribution only contains perturbed 
    non-secretory (ductal) cells.
    Then, each successive mixture distribution increases the probability of
    a secretory cell, simulating the cell arrival statistics
    as we sweep from right to left over a section of the pancreas for this data.
    
    % For stability of our results, we compute the average regret which is  an average over $10$ independent identical sub-instances. Each sub-instance consists of a stream of $10$-dimensional RNA expressions that is generated by randomly perturbing the pancreatic RNA expression data by a factor of $\pm 10\%$. The length of the stream is around $15,000$ examples. 

    % The non-secretory cells (``ductal cells'') are sampled and their expression levels perturbed by at most $10\%$ and given a class value of $0$. On the other end of the spectrum, the secretory cells (``alpha'', ``beta'', and ``delta'' cells) are given class values of $0.75$, $1.0$ and $0.5$ respectively.  We train for $14,000$ examples, and after every $2,250$ examples the input distribution jumps to a different distribution. Each successive distribution has about $10\%$ more secretory cells vs non-secretory cells than the previous distribution.\footnote{This roughly simulates the arrival of cells in right to left order from a section of the pancreas.}
    
    \item \emph{Neural network:}
    The input is a $10$-dimensional vector of RNA expression levels for the cell.
    We then use a feedforward neural network with five hidden layer and dimensions
    $(64, 32, 16, 8, 4)$.
    Each hidden layer uses an ELU activation,
    and the last $4$-dimensional embedding after activation are the logits
    for the cell type.

    % We set up the simulation as a multi-class classification problem on our single hidden layer network ($10$ neurons width, ELU activation\footnote{The ReLU suffers from a dead neuron problem in such small symmetric network with positive inputs, hence our choice of ELU.} in the hidden layer and sigmoid activation for the output neuron). 
    
    \item \emph{Loss and optimizer:}
    We use categorical cross entropy loss
    with \texttt{from\_logits=true} for stability.
    Each step uses a batch of $B_t = 64$
    new examples to simulate the data stream.
    We optimize this model in an online manner using
    Adam~\citep{kingma2014adam}
    for different initial learning rates
    and by optionally resetting its parameters at the beginning of a distribution shift.
    We plot the cumulative regret in \Cref{fig:flow_cytometry} (right),
    where the regret for each step is defined in \eqref{eqn:regret_def}.

    %\item \emph{Evaluation metric:} We plot the regret calculated from a $10\%$ holdout of training data in Figure~\ref{fig:flow-cyt-summary}. Each of the cell types is completely characterized by the chosen 10 different gene expressions. Therefore, the validation loss in absence of distribution shift will quickly fall to $0$ for SGD, even with a simple one layer neural net. Therefore, the regret without distribution shift is close to $0$, and the validation loss at any time-point minus any initial loss (after training on the first few hundred examples) should be close to regret at that time-point. We compute this approximation of regret, average it (as mentioned above), and plot it vs the number of examples seen in Figure~\ref{fig:flow-cyt-summary}. 
\end{itemize}

\subsubsection{Results}
We draw several conclusions from this experiment.
First, while larger learning rates are often better for minimizing the regret of an online
SGD-based system, there is a normally a sweet spot before the
first step size that causes the SGD to diverge.
In this experiment, an initial learning rate of $0.1$ for Adam
caused the model to diverge
but the total regret is minimized with an initial learning rate of $0.01$,
achieving less regret than $\eta_0 \in \{0.001, 0.003, 0.03\}$.
Second, resetting the Adam optimizer at the beginning of each distribution shift
(which increases its step size) allows us to achieve less
cumulative regret, as these models more quickly adapt to the new data distributions.
Finally, the models get stuck in local minima without
adaptive and increasing learning rate schedules, as evident by the
$\eta_0 = 0.03$ plots in \Cref{fig:flow_cytometry} (right),
which have different slopes in the final two phases.

% \begin{itemize}
%     \item We see that for constant learning rate schedules, the average regret decreases with increasing learning rate but after a certain point the higher learning rate is counterproductive. For this dataset derived from pancreatic RNA expressions,
%     the optimal constant learning rate is around $\eta_t = 0.005$.
%     \item The average regret of an adaptive learning rate schedule that is high only at time points when a shift in input distribution occurs is lower than the average regret of the optimal constant rate schedule, even though we keep the average learning rate same. 
%     \item Finally, if we vary the batch size but keep the frequency of jumps in input distribution constant, then we find that large batch sizes can be counterproductive.
%     This is because the model
%     sees fewer batches between two successive jumps in the input distribution. 
% \end{itemize}
\section*{Conclusion}

This work explores learning rate schedules that minimize regret
for online SGD-based learning in the presence of distribution shifts.
We derive a novel stochastic differential equation to approximate
the SGD path for linear regression with model shifts,
and we derive new adaptive schedules for general convex and non-convex losses
that minimize regret upper bounds.
These learning rate schedules can increase in the presence of distribution shifts
and allow for more aggressive optimization.

For future works, we propose extending our SDE framework to develop adaptive adjustment schemes for other hyperparameters in SGD variants such as Polyak averaging \citep{polyak1992acceleration}, SVRG \citep{johnson2013accelerating},
and elastic averaging SGD \citep{zhang2015deep},
as well as deriving effective adaptive momentum parameter adjustment policies.
We also propose studying a ``model hedging'' question
to quantify how neutral a model should remain at a given time
to optimally trade off between
underfitting the current distribution
and
being able to quickly adapt to a (possibly adversarial) future distribution.
We believe this area of designing adaptive learning rate schedules
is a fruitful and exciting area that combines control theory,
online optimization, and large-scale recommender systems~\citep{anil2022factory,coleman2023unified}.

\bibliographystyle{abbrvnat}
\bibliography{references}

\newpage
\appendix
\section{Proof of theorems and technical lemmas for linear regression}\label{sec:proofs}
\subsection{Proof of Proposition~\ref{pro:approx}}\label{proof:approx}
The integration form of the stochastic differential equation~\eqref{eq:SDE} reads as
\begin{align}\label{eq:X-proc}
X(\tau) = X_0+ Y_0- \int_0^\tau \zeta(s)X(s)\de s -Y(\tau)+ \int_0^\tau \frac{\zeta(s)}{\sqrt{\nu(s)}} D_s^{1/2}\de W(s)\,,
\end{align}
where $D_s = ((\|X(s)\|^2+\sigma^2)I + X(s)X(s)^\sT)$.
We start by proving some useful bounds on the solution of $X(\tau)$ process.
\begin{lemma}\label{lem:approx}
Consider the process $X(\tau)$ given by \eqref{eq:X-proc} with initialization $X_0$ satisfying $\|X_0\|\le K$. Under Assumptions {\sf A1-A2}, with probability at least $1-e^{-u^2}$ we have
\begin{align}\label{eq:XtauB}
\sup_{\tau\in[0,T]}\|X(\tau)\| \le C\sqrt{T}(\sqrt{d}+u)\exp\Big[C\Big(T^2+(\sqrt{d}+u)^2T\Big)\Big]\,.
\end{align}
and 
\begin{align}\label{eq:DXtauB}
\sup_{t\in [0,T/\eps]\cap\mathbb{Z}_{\ge 0}} \sup_{u\in[0,\eps]} \|X(t\eps+u) - X(t\eps)\|\le C'\sqrt{T\eps}(\sqrt{d}+u)^2\exp\Big[C\Big(T^2+(\sqrt{d}+u)^2T\Big)\Big]\,,
\end{align}
for any fixed $u>0$, and constants $C= C(K,\sigma)$,
$C' = C'(K,\sigma,\Gamma)$. 
\end{lemma}

\begin{proof}[Proof of Lemma~\ref{lem:approx}]
Define $V(\tau): = \int_0^\tau \frac{\zeta(s)}{\sqrt{\nu(s)}} D_s^{1/2}\de W(s)$. We have 
$${\rm Cov}(V(\tau)) = \int_0^\tau \frac{\zeta(s)^2}{\nu(s)} D_s \de s\,,$$
so then
\begin{align}
\|{\rm Cov}(V(\tau))\|_{\rm op}\le K^2 \int_0^\tau\|D_s\|_{\rm op}\de s
\le
A_\tau:=K^2 \int_0^\tau (2\|X_s\|^2+\sigma^2)\de s\,.
\end{align}

Note that $\exp{\alpha \|V(\tau)\|^2}$ is a submartingale, and by virtue of Doob’s martingale inequality, we have
\[
\prob\Big(\sup_{\tau\le T}\|V(\tau)\| \ge \lambda\Big)
\le \E[\exp\{\alpha\|V(T)\|/2\}] \exp\{-\alpha \lambda^2/2\}\le (1-A_T\alpha)^{-d/2} \exp\{-\alpha \lambda^2/2\}\,.
\]
Take $\alpha = 1/(2A_T)$ and $\lambda = 2\sqrt{A_T}(\sqrt{d}+u)$ to obtain
\begin{align}\label{eq:VB}
    \prob\left(\sup_{\tau\le T}\|V(\tau)\| \ge 2\sqrt{A_T}(\sqrt{d}+u) \right)
\le 2^{d/2}\exp(-(\sqrt{d}+u)^2) \le e^{-u^2}\,.
\end{align}
Using \eqref{eq:X-proc} and recalling Assumptions {\sf A1-A2}, we get
\begin{align*}
    \|X(\tau)\| &\le \|X_0\|+ \|Y_0\|+ \|Y_\tau\| + \int_0^{\tau} \zeta(s) \|X(s)\|\de s + \|V(\tau)\|\\
    &\le 3K + \int_0^{\tau} K \|X(s)\|\de s + \|V(\tau)\|\,.
\end{align*}
We next use the inequality $(a+b+c)^2\le 3(a^2+b^2+c^2)$ to get
\begin{align*}
    \|X(\tau)\|^2 &\le 27K^2 + 3K^2\Big(\int_0^{\tau}  \|X(s)\|\de s\Big)^2 + 3\|V(\tau)\|^2\\
    &\le 27K^2 + 3K^2 \tau \int_0^{\tau} \|X(s)\|^2\de s + 3\|V(\tau)\|^2\,,
\end{align*}
where in the second line we used Cauchy--Shwarz inequality. Define $\Delta_T = \sup_{\tau\le T}\|X(\tau)\|^2$.
Taking the supremum over $\tau\le T$ of both sides of the previous inequality and using the bound \eqref{eq:VB}, we arrive at
\begin{align*}
\Delta_T&\le 27K^2 + 3K^2T \int_0^{T} \Delta_s\de s + 12A_T(\sqrt{d}+u)^2\\
&\le 27K^2 + 3K^2T \int_0^{T} \Delta_s\de s + 12A_T(\sqrt{d}+u)^2\\
&\le 27K^2 + 3K^2T \int_0^{T} \Delta_s\de s + 12(2K^2\int_0^T\Delta_s \de s+\sigma^2TK^2)(\sqrt{d}+u)^2\\
&= 27K^2+ 12\sigma^2TK^2(\sqrt{d}+u)^2+
(3T+24(\sqrt{d}+u)^2)K^2 \int_0^T \Delta_s \de s.
\end{align*}
Using Gronwall’s inequality, the above relation implies that
\[
\Delta_T\le K^2 (27+ 12\sigma^2T(\sqrt{d}+u)^2)
\exp((3T+24(\sqrt{d}+u)^2)K^2T).
\]
Taking square root of both sides and using $\sqrt{a+b}\le \sqrt{a}+\sqrt{b}$, we get
\[
\sup_{\tau\le T}\|X(\tau)\|\le K (\sqrt{27}+ \sqrt{12T}\sigma(\sqrt{d}+u)) \exp((3T+24(\sqrt{d}+u)^2)K^2T/2)\,,
\]
which completes the proof of~\eqref{eq:XtauB}.

We next proceed with proving~\eqref{eq:DXtauB}.
Define $\tilde{\Delta}(t,\eps) = \sup_{h\in[0,\eps]}\|X(t\eps+h)- X(t\eps)\|$.
Using \eqref{eq:X-proc}, we have
\begin{align}
\tilde{\Delta}(t,\eps) &\le \sup_{h\in[0,\eps]}\left\{ \Big\|\int_{t\eps}^{t\eps+h} \zeta(s) X(s)\de s\Big\|
+ \|Y(t\eps+h)-Y(t\eps)\| + \Big\|\int_{t\eps}^{t\eps+h} \frac{\zeta(s)}{\sqrt{\nu(s)}} D_s^{1/2} \de W(s)\Big\|\right\}\nonumber\\
 &\le K \eps\sup_{s\le T} \|X(s)\| + \sup_{h\in[0,\eps], \tau\in[t\eps, t\eps+h]} \|Y'(\tau)\| h+ \sup_{h\in[0,\eps]} \|V(t,h,\eps)\|\nonumber\\
 &\le K \eps\sup_{s\le T} \|X(s)\| + \Gamma+ \sup_{h\in[0,\eps]} \|V(t,h,\eps)\|
 \,,\label{eq:tD-B}
\end{align}
with $V(t,h,\eps): = \int_{t\eps}^{t\eps+h} \frac{\zeta(s)}{\sqrt{\nu(s)}} D_s^{1/2} \de W(s)$. In the last step, we used Assumption {\sf A2}, by which 
$\|Y'(\tau)\|\le \Gamma/\eps$.
Similar to the derivation of \eqref{eq:VB}, we have
\[
\prob\left(\sup_{t\in[0,T/\eps]\cap \mathbb{Z}_{\ge 0}} \sup_{h\in[0,\eps]}\|V(t,h,\eps)\| \ge 2\sqrt{B_\eps}(\sqrt{d}+u) \right)
\le 2^{d/2}\exp(-(\sqrt{d}+u)^2) \le e^{-u^2}\,,
\]
with $B_\eps:= \sup_{h\le \eps} K^2 \int_{t\eps}^{t\eps+h} (2\|X(s)\|^2+\sigma^2)\de s$. Plugging in~\eqref{eq:tD-B}, we have
\begin{align}
  \tilde{\Delta}(t,\eps) &\le  K \eps\sup_{s\le T} \|X(s)\| + 2 K\sqrt{(2\sup_{s\le T}\|X(s)\|^2+\sigma^2)\eps}\; (\sqrt{d}+u)\nonumber\\
  &= \Gamma+ K\left(\eps+2\sqrt{2\eps}(\sqrt{d}+u)\right) (\sup_{s\le T} \|X(s)\|) + 2K\sigma\sqrt{\eps} (\sqrt{d}+u)\nonumber\\
  &\le C'\sqrt{T\eps}(\sqrt{d}+u)^2\exp\Big[C\Big(T^2+(\sqrt{d}+u)^2T\Big)\Big]\,.
\end{align}
This concludes the proof of Equation~\eqref{eq:DXtauB}.

We next rewrite the stochastic gradient descent update as follows:
 \begin{align}
     \theta_{t+1} &= \theta_t - \eta_t \frac{1}{B_t}\sum_{k=1}^{B_t}\nabla\ell(\theta_t,z_{k,t})\nonumber\\
     &= \theta_t + \eta_t \frac{1}{B_t}\sum_{k=1}^{B_t} (y_{tk}-\<x_{tk},\theta\>) x_{tk}\nonumber\\
     %&= \theta_t -\eta_t \nabla\ellbar_t(\theta_t)-
     %\eta_t \left(\frac{1}{B_t}\sum_{k=1}^{B_t}\nabla\ell(\theta_t,z_{k,t}) - \nabla\ellbar_t(\theta_t)\right)\nonumber\\
      &= \theta_t +\eta_t (\theta^*_t - \theta_t) +\frac{\eta_t}{B_t}\sum_{k=1}^{B_t} \left((x_{tk}x_{tk}^\sT - I)(\theta^*_t - \theta_t)+\eps_{tk} x_{tk}\right)\nonumber\\
      & =\theta_t +\eta_t (\theta^*_t - \theta_t) -\eta_t\xi_t \,, \label{eq:zeta}
 \end{align}
 where the noise term $\xi_t$ has mean zero, given that the data points $z_{t,k}$ are sampled independently at each step $t$. 

Note that   
 $\xi_t$ in~\eqref{eq:zeta} is the average of $B_t$ zero mean variables and thus can be approximated by a normal distribution with covariance $(1/B_t) D_t$, with
 \begin{align}
 D_t &= \left\{\E\left[(x_{tk} x_{tk}^\sT - I)(\theta^*_t - \theta_t)(\theta^*_t - \theta_t)^\sT (x_{tk} x_{tk}^\sT - I)^\sT\right] +\sigma^2 I\right\}\nonumber\\
 &= \left((\|\theta^*_t-\theta_t\|^2+\sigma^2)I +  (\theta^*_t-\theta_t)(\theta^*_t-\theta_t)^\sT \right)\,,
 \end{align} 
 where the above identity follows from Lemma~\ref{lem:stein}. We let $\xi_t = -D_t^{1/2}g_t$ with $g_t\sim\normal(0,I_d)$. Iterating update \eqref{eq:zeta} recursively, we have
 \begin{align}
     \theta_t - \theta^*_t &= \theta_0 - \theta^*_t  +\sum_{\ell=0}^{t-1}  \eta_\ell(\theta^*_\ell-\theta_\ell)+ \sum_{\ell=0}^{t-1} \frac{\eta_\ell}{\sqrt{B_\ell}} D_{\ell}^{1/2}g_\ell\nonumber\\
     &=\theta_0 -\theta^*_t +\eps\sum_{\ell=0}^{t-1}  \zeta(\ell \eps)(\theta^*_\ell-\theta_\ell)+ \int_0^{t\eps} \frac{s\zeta([s])}{\sqrt{s\nu([s])}} D_{[s]}^{1/2} \frac{\de W(s)}{\sqrt{s}}\nonumber\\
     &=\theta_0  - Y(\eps t)  +\eps\sum_{\ell=0}^{t-1}  \zeta(\ell \eps)(\theta^*_\ell-\theta_\ell)+ \int_0^{t\eps} \frac{\zeta([s])}{\sqrt{\nu([s])}} D_{[s]}^{1/2} \de W(s)\,,\label{eq:SGD-full}
 \end{align}
 where we adopt the notation $[s] = \eps \lfloor s/\eps\rfloor$, and $W(s)$ represents the standard Brownian motion.
 
 We take the difference of \eqref{eq:X-proc} and \eqref{eq:SGD-full}.
 Since $\theta_0 = \theta_0 - \theta^*_0 + \theta^*_0 = X_0 + Y_0$, for $\tau\in \mathbb{Z}_{\ge 0}\eps\cap[0,T]$, we have:
 \begin{align}
     \|(\theta_{\tau/\eps}- \theta^*) - X(\tau)\|
     &\le \left\|\eps\sum_{\ell=0}^{\tau/\eps-1}  \zeta(\ell \eps)(\theta_\ell- \theta^*_\ell) - \int_0^\tau \zeta(s)X(s)\de s\right\| \notag\\
     &\hspace{0.45cm}+ \left\|\int_0^\tau \Big(\frac{\zeta([s])}{\sqrt{\nu([s])}} - \frac{\zeta(s)}{\sqrt{\nu(s)}} \Big)\de W(s) \right\|.\label{eq:app1}
 \end{align}
 We first treat the first term. We have
 \begin{align*}
     &\eps\sum_{\ell=0}^{t-1}  \zeta(\ell \eps)(\theta_\ell- \theta^*_\ell) - \int_0^\tau \zeta(s)X(s)\de s \nonumber\\
     &= \int_0^{\tau} \zeta([s]) (\theta_{\lfloor s/\eps\rfloor} - \theta^*_{\lfloor s/\eps\rfloor}) - \int_0^\tau \zeta(s)X(s)\de s\\
     &=\int_0^{\tau} \zeta([s]) \Big(\theta_{\lfloor s/\eps\rfloor} - \theta^*_{\lfloor s/\eps\rfloor} -X([s])\Big) \de s + \int_0^\tau \zeta([s]) (X([s]) - X(s))\de s
     + \int_0^\tau (\zeta([s])-\zeta(s)) X(s)\de s.
 \end{align*}
 We have
 \begin{align}\label{eq:aux1}
    \left\| \int_0^\tau \zeta([s]) (X([s]) - X(s))\de s\right\|
     \le K\tau \sup_{t\in [0,T/\eps]\cap\mathbb{Z}_{\ge 0}} \sup_{h\in[0,\eps]} \|X(t\eps+h) - X(t\eps)\|\,.
 \end{align}
 Also,
 \begin{align}\label{eq:aux2}
   \left\| \int_0^\tau (\zeta([s])-\zeta(s)) X(s)\de s\right\|
   \le K\eps \tau \sup_{\tau\in[0,T]}\|X(\tau)\|\,.
 \end{align}
 Note that the right-hand side of~\eqref{eq:aux1} and \eqref{eq:aux2} are bounded in Lemma~\ref{lem:approx}.
 
 We next bound the second term on the right-hand side of~\eqref{eq:app1}.
 Define
 \[
 E(\tau): = \int_0^\tau \Big(\frac{\zeta([s])}{\sqrt{\nu([s])}} - \frac{\zeta(s)}{\sqrt{\nu(s)}} \Big)\de W(s).
 \]
 Note that $E(\tau)~\sim\normal(0,\alpha^2 I_d)$,
 where
 \[
 \alpha^2 = \int_0^\tau \Big(\frac{\zeta([s])}{\sqrt{\nu([s])}} - \frac{\zeta(s)}{\sqrt{\nu(s)}} \Big)^2\de s\le K ^2\eps \tau\,,
 \]
 using Assumption {\sf A2} by which $\|\zeta/\sqrt{v}\|_{\rm Lip}\le K$. By applying Doob's inequality to the martingale $\exp(\frac{1}{2\tau}\|E(\tau)\|)$, similar to derivation of \eqref{eq:VB}, we obtain
 \begin{align}\label{eq:E}
     \prob\left(\sup_{\tau\le T} \|E(\tau)\|\ge 2K \sqrt{\eps T}(\sqrt{d}+u)\right)\le e^{-u^2/2}\,.
 \end{align}

\noindent 
Now we define 
 $$\Delta(\tau) := \sup_{t\in[0,\tau/\eps]\cap\mathbb{Z}_{\ge 0}}\|X(t\eps) - (\theta_t-\theta^*_t)\|\,.$$
 Using Lemma~\ref{lem:approx} to bound \eqref{eq:aux1} and \eqref{eq:aux2} and then combining that with \eqref{eq:E} into \eqref{eq:app1} we arrive at
 \begin{align}
     \Delta(\tau) &\le K\int_0^\tau \Delta(s)\de s+ K\tau C'\sqrt{T\eps}(\sqrt{d}+u)^2\exp\Big[C\Big(T^2+(\sqrt{d}+u)^2T\Big)\Big]\nonumber\\
     &\hspace{0.45cm}+ K\eps \tau C\sqrt{T}(\sqrt{d}+u)\exp\Big[C\Big(T^2+(\sqrt{d}+u)^2T\Big)\Big] + 2K\sqrt{\eps T}(\sqrt{d}+u)\nonumber\\
     &\le K\int_0^\tau \Delta(s)\de s
     +C'' T^{3/2} \sqrt{\eps}(\sqrt{d}+u)^2\exp\Big[C\Big(T^2+(\sqrt{d}+u)^2T\Big)\Big].
 \end{align}
 Using Gronwall's inequality we obtain
 \begin{align}
     \Delta(T)\le C'' T^{3/2} \sqrt{\eps}(\sqrt{d}+u)^2\exp\Big[C\Big(T^2+(\sqrt{d}+u)^2T\Big)+KT\Big]\,.
 \end{align}
 We derive the final claim by noting that
 \begin{align}
     \sup_{t\in[0,\tau/\eps]\cap\mathbb{Z}_{\ge 0}}\Big|\|X(t\eps)\|^2 - \|\theta_t-\theta^*_t\|^2\Big|
     &\le \Delta(T)^2+ 2\Delta(T) \sup_{t\in[0,\tau/\eps]\cap\mathbb{Z}_{\ge 0}} \|X(t\eps)\|\nonumber\\
     &\le C_1 \sqrt{\eps}(\sqrt{d}+u)^4 T^3 \exp\Big[C_2\Big(T^2+(\sqrt{d}+u)^2T\Big)\Big]\,,
 \end{align}
 for some constants $C_1, C_2$, depending on $K,\sigma, \Gamma$. This completes the proof. 
\end{proof}
 
%===========================================
\subsection{Proof of Theorem~\ref{pro:SDE}}\label{sec:Ito-SDE}

Recall the SDE for process $X(\tau)$ given by
\[
\de X(\tau)=-(\zeta(\tau) X(\tau) +Y'(\tau))\de\tau+ \nonumber\\
     \frac{\zeta(\tau)}{\sqrt{\nu(\tau)}} \left((\|X(\tau)\|^2+\sigma^2)I + X(\tau)X(\tau)^\sT\right)^{1/2}\; \de W(\tau)\,,
\]
Let $m_\tau:=\E[X(\tau)]$. Taking expectation of the above SDE, we obtain
\[
 m_\tau'= - \zeta(\tau) m_\tau - Y'(\tau)\,.
\]

Next we define the stochastic process $Z(\tau) = \|X(\tau)\|^2$. By Ito's lemma (cf. Lemma~\ref{lem:Ito}), we have
\begin{align*}
\de Z(\tau) =& \left(-2\zeta(\tau)\|X(\tau)\|^2 - 2X(\tau)^\sT Y'(\tau) + \frac{\zeta(\tau)^2}{\nu(\tau)} \left((d+1) \|X(\tau)\|^2+d \sigma^2\right)\right) \de\tau \nonumber\\
&+ 2 \frac{\zeta(\tau)}{\sqrt{\nu(\tau)}} X(\tau)^\sT\left(( \|X(\tau)\|^2+\sigma^2)I+ X(\tau)X(\tau)^\sT\right)^{1/2}\; \de W(\tau)\,.
\end{align*}
Taking expectation of both sides, we arrive at the following ODE for $v_\tau = \E[Z(\tau)] = \E[\|X(\tau)\|^2]$:
\begin{align}
    v'_\tau &= -2\zeta(\tau)v_\tau - 2m_\tau^\sT  Y'(\tau)
    +\frac{\zeta(\tau)^2}{\nu(\tau)} ((d+1)v_\tau+ d\sigma^2)\nonumber\\
    &=\left((d+1)\frac{\zeta(\tau)^2}{\nu(\tau)} - 2\zeta(\tau)\right)v_\tau + \frac{\zeta(\tau)^2}{\nu(\tau)}\sigma^2 d- 2m_\tau^\sT Y'(\tau)\,.
\end{align}
%==========================
\subsection{Proof of Theorem \ref{thm:optimal-policy}}

We start by giving a brief overview of the Hamilton--Jacobi--Bellman (HJB) equation~\cite{bellman1956dynamic}. 

Consider the following value function:
\begin{align}
    V(z(\tau_0),\tau_0) = \min_{\zeta:[\tau_0,T]\to \mathcal{A}} \int_{\tau_0}^T C(z(\tau),\zeta(\tau))\de \tau + D(z(T))\,,
\end{align}
where $z(\tau)$ is the vector of the system state, $\zeta(\tau)$, for $\tau\in[\tau_0,T]$ is the control policy we aim to optimize over and takes value in a set $\mathcal{A}$,   $C(\cdot)$ is the scalar cost function and $D(\cdot)$ gives the bequest value at the final state $z(T)$.

Suppose that the system is also subject to the constraint
\begin{align}\label{eq:HJB-con}
\frac{\de}{\de \tau}{z}(\tau) = \Phi(z(\tau),\zeta(\tau))\,, \quad \forall \tau\in[\tau_0,T]\,,
\end{align}
with $\Phi$ describing the evolution of the system state over time. The dynamic programming principle allows us to derive a recursion on the value function $V$, in the form of a partial differential equation (PDE).
Namely, the the Hamilton--Jacobi--Bellman PDE is given by
\begin{align}
\partial_\tau V(z,\tau) + \min_{\zeta\in \mathcal{A}} \left[ \partial_z V(z,\tau)\cdot  \Phi(z,\zeta)  + C(z,\zeta) \right] = 0\,,\label{eq:HJB}\\
\text{subject to}\quad V(z,T) = D(z)\,.\nonumber
\end{align}
The above PDE can be solved backward in time and then the optimal control $\zeta^*(\tau)$ is given by 
\begin{align}\label{eq:zeta*}
\zeta^*(\tau) = \arg\min_{\zeta\in \mathcal{A}} \left[ \partial_z V(z(\tau),\tau) \cdot \Phi(z(\tau),\zeta)  + C(z(\tau),\zeta) \right]\,.
\end{align}

We are now ready to prove the claim of Theorem \ref{thm:optimal-policy}, using the HJB equation.  

Consider $\tilde{v}_\tau$ as the system state at time $\tau$ ( i.e., $z(\tau) = \tilde{v}_\tau$), and the cost function $C(\tilde{v}_\tau, \zeta(\tau)) = \tilde{v}_\tau$. Also set $D(\cdot)$ to be the zero everywhere. The control variable $\zeta(\tau)$ takes values in $\mathcal{A} = [0,1]$.

The function $\Phi(\cdot,\cdot)$ in~\eqref{eq:HJB-con} is given by~\eqref{eq:tv}, which we recall here:

\[
\Phi(\tilde{v}_\tau,\zeta): = \Big((d+1)\frac{\zeta^2 }{\nu(\tau)}-2\zeta\Big) \tilde{v}_\tau 
     + \frac{\zeta^2}{\nu(\tau)}\sigma^2 d + 2\|Y'(\tau)\|\sqrt{\tilde{v}_\tau}\,.
\]

Note that in our case, the cost function $C$ does not depend on $\zeta(\tau)$. Also, it is easy to see that $\partial_z V(\tilde{v}_\tau,\tau)>0$ because larger $\tilde{v}_\tau$ means we are further from the sequence of models and so the minimum cost achievable in tracking the sequence of models will be higher. Therefore, \eqref{eq:zeta*} reduces to
\[
\zeta^*(\tau) = \arg\min_{\zeta\in[0,1]} \Phi(\tilde{v}_\tau,\zeta).
\]
Since $\Phi$ is quadratic in $\zeta$, solution to the above optimization has a closed form given by
\[
\zeta^*(\tau) = \min\left\{1, \left(\frac{d+1}{\nu(\tau)}\tilde{v}_\tau + \frac{\sigma^2d}{\nu(\tau)}\right)^{-1} \tilde{v}_\tau\right\}\,,
\]
which completes the proof.

%==============================
\subsection{Proof of Lemma~\ref{lem:no-shift-Lin}}

Substituting for $\zeta(\tau)$ from~\eqref{eq:policy}, it is easy to verify that $\tilde{v}'(\tau)\le0$ and so $\tilde{v}(\tau)$ is decreasing in $\tau$.

Define the shorthand ${\sf a}: = (d+1)/\nu(\tau)$ and ${\sf b}:= \sigma^2d/\nu(\tau)$.
Note that if $\tilde{v}_\tau\ge {\sf b}/(1-{\sf a})$,
% \Big(\frac{d+1}{\nu(\tau)} -1\Big) \tilde{v}(\tau) + \frac{\zeta(\tau)^2}{\nu(\tau)}\sigma^2 d\le 0,
then by \eqref{eq:policy}, $\zeta(\tau) = 1$ and in this case ODE~\eqref{eq:tv} reduces to $\tilde{v}_\tau' = ({\sf a}-2)\tilde{v}_\tau+{\sf b}$, with the solution
\[
\tilde{v}_\tau = \left(\tilde{v}_0 + \frac{{\sf b}}{{\sf a}-2}\right) e^{({\sf a}-2)\tau}-\frac{{\sf b}}{{\sf a}-2}\,.
\]
However, the above solution is valid until $\tilde{v}_\tau\ge {\sf b}/(1-{\sf a})$, which is the assumption we started with, which using the above characterization is equivalent to 
\[
\tau\le \tau_*:= \left[\frac{1}{2-{\sf a}}\log\left((1-{\sf a}) \left(\tilde{v}_0 \frac{2-{\sf a}}{{\sf b}} - 1\right)\right)\right]_+\,.
\]
For $\tau>\tau_*$, we have $\tilde{v}_{\tau}\le {\sf b}/(1-{\sf a})$ and so $\zeta(\tau) = \tilde{v}_\tau/({\sf a}\tilde{v}_\tau + {\sf b})$ by~\eqref{eq:policy}. In this case, ODE~\eqref{eq:tv} reduces to
\[
\tilde{v}'_\tau = -\frac{\tilde{v}_{\tau}^2}{{\sf a}\tilde{v}_{\tau}+{\sf b}}\,.
\]
By rearranging the terms and integrating, the solution to above ODE satisfies
\begin{align}\label{eq:ODE2}
{\sf a} \ln\Big(\frac{1}{\tilde{v}_\tau}\Big) + \frac{{\sf b}}{\tilde{v}_{\tau}} = \tau+ C\,,
\end{align}
where $C$ can be obtained by the continuity condition of $\tilde{v}_{\tau}$ at $\tau_*$, i.e.,
\[
C = {\sf a}\ln\Big(\frac{1-{\sf a}}{{\sf b}}\Big) +1-{\sf a} -\tau_*\,.
\]
From~\eqref{eq:ODE2} we observe that as $\tau\to \infty$, $\tilde{v}_\tau\to 0$ and the term ${\sf b}/\tilde{v}_\tau$ becomes dominant by which we obtain
\[
\lim_{\tau\to \infty} \frac{\tilde{v}_\tau}{\frac{{\sf b}}{\tau+C}} = 1\,.
\]
In addition, invoking definition of optimal policy $\zeta(\tau)$, we obtain
\[
\lim_{\tau\to \infty} \frac{\zeta(\tau)}{\frac{1}{{\sf a}+C+\tau}} = 1\,,
\]
which completes the proof.

%=============================
\section{Proof of theorems and technical lemmas for convex loss}
\subsection{Proof of Theorem~\ref{thm:convex-reg}}
We define the shorthand $D_t^2 = \|\theta^*_t - \theta_t\|^2$ and let $v_t = \theta^*_t - \theta^*_{t+1}$ be shifts in the optimal models. We also define the shorthand 
\[
\nabla\ell^B_t(\theta_t):=
\frac{1}{B_t} \sum_{k=1}^{B_t} \nabla\ell(\theta_t,z_{t,k})\,.
\]
Since projection on a convex set is contraction, we have
\[
\|\Pi_{\Theta}(u) - w\|\le \|u - w\|\,,
\]
for any $w\in \Theta$. Using this property, we have
\begin{align*}
D_{t+1}^2 &= \|\Pi_{\Theta}(\theta_t-\eta_t \nabla\ell^B_t(\theta_t)) - \theta^*_{t+1} \|^2\\
&=\|\Pi_{\Theta}(\theta_t-\eta_t \nabla\ell^B_t(\theta_t))-\theta^*_t+ \theta^*_t - \theta^*_{t+1}\|^2\\
&= \|\Pi_{\Theta}(\theta_t-\eta_t \nabla\ell^B_t(\theta_t))-\theta^*_t\|^2+\|v_t\|^2
+2\<v_t, \Pi_{\Theta}(\theta_t-\eta_t \nabla\ell^B_t(\theta_t))-\theta^*_t\> \\
&\le \|\theta_t-\eta_t \nabla\ell^B_t(\theta_t)-\theta^*_t\|^2+\|v_t\|^2
+2\<v_t, \Pi_{\Theta}(\theta_t-\eta_t \nabla\ell^B_t(\theta_t))-\theta^*_t\>\\
&= D_t^2 - 2\eta_t\<\nabla \ell^B_t(\theta_t), \theta_t - \theta_t^*\> +\|v_t\|^2
+2\<v_t, \Pi_{\Theta}(\theta_t-\eta_t \nabla\ell^B_t(\theta_t))-\theta^*_t\> +\eta_t^2\|\nabla \ell^B_t(\theta_t)\|^2.
\end{align*}

% \begin{align*}
% D_{k+1} &= \|\theta_k -\theta_k^*-\eta_k \nabla\ell(\theta_k,z_k)+\theta^*_k - \theta^*_{k+1}\|^2\\
% &= D_k - 2\eta_k\<\nabla \ell(\theta_k,z_k), \theta_k - \theta_k^*\> +\|v_k\|^2
% +\<v_k, \theta_k -\theta_k^*-\eta_k \nabla\ell(\theta_k,z_k)\> +\eta_k^2\|\nabla \ell(\theta_k,z_k)\|^2
% \end{align*}
Define
\[
\delta_t: = \nabla\ell^B_t(\theta_t) -\nabla \ellbar_t(\theta_t)\,,
\]
as the difference between the gradient of the expected loss (at step $t$) and the gradient of the batch average loss at that step.

Writing the above bound in terms of this notation, we get
\begin{align}
D_{t+1}^2 
&\le D_t^2 - 2\eta_t\<\nabla \ellbar_t(\theta_t)+\delta_t, \theta_t - \theta_t^*\> +\|v_t\|^2
+2\<v_t, \Pi_{\Theta}(\theta_t-\eta_t \nabla\ell^B_t(\theta_t))-\theta^*_t\>\nonumber\\ &\quad +\eta_t^2\Big(\|\nabla \ellbar_t(\theta_t)\|^2+ \|\delta_t\|^2+2\<\delta_t,\nabla\ellbar_t(\theta_t)\> \Big)\,.\label{eqn:common}
\end{align}

By \citet[Lemma 4]{zhou2018fenchel} for any $L$-smooth convex function $f$, we have
\begin{align}\label{eq:conv-prop1}
\frac{1}{L} \|\nabla f(y)-\nabla f(x)\|^2\le \<\nabla f(y)-\nabla f(x), y-x\>\,.
\end{align}

Since the loss function $\ell(\theta,z)$ is convex, the expected loss functions $\ellbar_t(\theta)$ are also convex for $t=1,\dotsc, T$.
Using~\eqref{eq:conv-prop1} together with the fact that $\nabla\ellbar_t(\theta^*_t) = 0$ by optimality of $\theta^*_t$, we get
\begin{equation}\label{eqn:cvx1}
\frac{1}{L}\|\nabla\ellbar_t(\theta_t)\|^2\le \<\nabla \ellbar_t(\theta_t), \theta_t - \theta_t^*\>\,.
\end{equation}

Using the above bound, we obtain
\begin{align*}
D_{t+1}^2 
&\le D_t^2 - (2\eta_t-L\eta_t^2)\<\nabla \ellbar_t(\theta_t), \theta_t - \theta_t^*\> +\|v_t\|^2
+2\<v_t, \Pi_{\Theta}(\theta_t-\eta_t \nabla\ell^B_t(\theta_t))-\theta^*_t\>\\ &\quad +\eta_t^2\|\delta_t\|^2- 2\eta_t \<\delta_t,\theta_t - \theta_t^*-\eta_t\nabla \ellbar_t(\theta_t)\>.
\end{align*}
Recall our assumption $\eta_t\le 2/L$.
Using the convexity of $\ellbar_k$, we have 
\begin{equation}\label{eqn:cvx2}
\ellbar_t(\theta_t)-\ellbar_t(\theta^*_t)\le \<\nabla \ellbar_t(\theta_t), \theta_t - \theta^*_t\>,
\end{equation}
which along with the above bound implies that
\begin{align*}
D_{t+1}^2 
&\le D_t^2 - (2\eta_t-L\eta_t^2)(\ellbar_t(\theta_t)-\ellbar_t(\theta^*_t)) +\|v_t\|^2
+2\<v_t, \Pi_{\Theta}(\theta_t-\eta_t \nabla\ell^B_t(\theta_t))-\theta^*_t\>\\ &\quad +\eta_t^2\|\delta_t\|^2- 2\eta_t \<\delta_t,\theta_t - \theta_t^*-\eta_t\nabla \ellbar_t(\theta_t)\> \,.
\end{align*}

 Note that $\Pi_{\Theta}(\theta_t-\eta_t \nabla\ell^B_t(\theta_t))-\theta^*_t = \theta_{t+1}-\theta^*_t$.
% Since $\theta_{t+1}, \theta^*_t\in \Theta$, we have $\|\theta_{t+1}-\theta^*_t\|\le D_{\max}$. Hence,
% \begin{align}\label{eqn:cvx3}
% D_{t+1}^2 
% &\le D_t^2 - (2\eta_t-L\eta_t^2)(\ellbar_t(\theta_t)-\ellbar_t(\theta^*_t)) +\gamma_t^2
% +2D_{\max}\gamma_t \nonumber\\ 
% &\quad +\eta_t^2\|\delta_t\|^2- 2\eta_t \<\delta_t,\theta_t - \theta_t^*-\eta_t\nabla \ellbar_t(\theta_t)\>\,. 
% \end{align}
We let $a_t: = 2\eta_t - L\eta_t^2 > 0$, and by rearranging the terms in the above equation we obtain
\begin{align}
    \ellbar_t(\theta_t)-\ellbar_t(\theta^*_t)
    \le \frac{D_t^2}{a_t} - \frac{D_{t+1}^2}{a_t}+
    \frac{\|v_t\|^2}{a_t}+\frac{2}{a_t}\<v_t, \theta_{t+1}-\theta^*_t\>+\frac{\eta_t^2\|\delta_t\|^2}{a_t} - \frac{2\eta_t}{a_t}\<\delta_t,\theta_t - \theta_t^*-\eta_t\nabla \ellbar_t(\theta_t)\>\,.\label{eq:Di-B}
\end{align}

We next note that $\theta_t, \theta^*_t, \eta_t$ are adapted to the filtration $\bz_{[t-1]}$, and therefore,
\[
\E[\<\delta_t,\theta_t - \theta^*_t-\eta_t\nabla\ellbar_t(\theta_t)\>|\bz_{[t-1]}] = \<\E[\delta_t|\bz_{[t-1]}],\theta_t - \theta^*_t-\eta_t\nabla \ellbar_t(\theta_t)\>= 0\,.
\]
%In addition by our assumption $\E[\|\delta_k\|^2]\le \sigma^2$.
Taking iterated expectations of both sides of \eqref{eq:Di-B} with respect to filtration $\bz_t$ (first conditional on $\bz_{[t-1]}$ and then with respect to $\bz_{[t-1]}$), we get
\begin{align}
\E[\reg_t] 
\le \E\left[\frac{D_t^2-D_{t+1}^2}{a_t} 
+ \frac{\sigma^2}{B_t}\frac{\eta_t^2}{a_t} + \frac{\|v_t\|^2}{a_t} + \frac{2}{a_t}\<v_t, \theta_{t+1}-\theta^*_t\>\right]\,,\label{eq:Di-B2}
\end{align}
with $\reg_t = \ellbar_t(\theta_t) - \ellbar_t(\theta^*_t)$.
Summing both sides over $t =1,\dotsc, T$, we obtain the desired result.
% \begin{align}\label{eqn:same-ub}
% &\E[\Reg(T)]= \sum_{t=1}^T \E[\reg_t]\nonumber\\
% &\le \sum_{t=1}^T\E\left[ \left(\frac{D_t^2}{a_t} - \frac{D_{t+1}^2}{a_{t}}\right) + \frac{\sigma^2\eta_t^2}{B_ta_t} + \frac{\gamma_t^2}{a_t} + 2D_{\max}\sum_{t=1}^T\frac{\gamma_t}{a_t}\right]\,.
% \end{align}

% Compare Equation~\ref{eqn:same-ub} above for the upper envelope to that for the lower envelope in Equation~\ref{eqn:same-lb}. However, we can further relax and simplify the upper bound slightly to obtain:
% \begin{align*}
% &\E[\Reg(T)]= \sum_{t=1}^T \E[\reg_t]\\
% &\le \E\left[\sum_{t=1}^T \left(\frac{D_t^2}{a_t} - \frac{D_{t+1}^2}{a_{t}}\right) + \sum_{t=1}^T \frac{\sigma^2\eta_t^2}{B_ta_t} + \sum_{t=1}^T\frac{\gamma_t^2}{a_t} + 2D_{\max}\sum_{t=1}^T\frac{\gamma_t}{a_t}\right]\\
%     &= \sum_{t=2}^T \E\left[D_t^2 \left(\frac{1}{a_t} - \frac{1}{a_{t-1}}\right)\right]
%     +\E\left[\frac{D_1^2}{a_1}-\frac{D_{T+1}^2}{a_T}\right] + \sum_{t=1}^T \E\left[\frac{\sigma^2\eta_t^2}{B_ta_t} + \frac{\gamma_t^2}{a_t} + 2D_{\max}\frac{\gamma_t}{a_t}\right]\\
%     &\le D_{\max}^2\sum_{t=2}^T \E\left[ \left(\frac{1}{a_t} - \frac{1}{a_{t-1}}\right)_+\right]
%     +D_{\max}^2\E\left[\frac{1}{a_1}\right]   +\sum_{t=1}^T \E\left[\frac{\sigma^2\eta_t^2}{B_ta_t} + \frac{\gamma_t^2}{a_t} + 2D_{\max}\frac{\gamma_t}{a_t}\right]\,,
% \end{align*}
% where for a scalar $x$, $x_+ = \max(x,0)$ indicates its positive part.

%===================================
\subsection{Proof of Proposition~\ref{propo:optimal-eta}}
Recall the optimization problem for $\eta^*$ given below:
\begin{align}\label{eq:etat*}
    \eta_t^*:= \arg\min_{0\le \eta\le \frac{1}{L}} D_{\max}^2
    \left(\frac{1}{2\eta-L\eta^2} - \frac{1}{2\eta_{t-1}-L\eta_{t-1}^2}\right)_+
    +\frac{\sigma^2}{B_t}\cdot\frac{\eta^2}{2\eta-L\eta^2}+ \frac{\gamma_t^2+2D_{\max}\gamma_t}{2\eta-L\eta^2}\,.
\end{align}
Note that the functions $1/(2\eta - L\eta^2)$ and $\eta^2/(2\eta - L\eta^2)$ are convex for $\eta\le 1/L$. Also the pointwise maximum of convex functions is convex, which implies that the objective function above is convex. With that, we first derive the stationary points of the objective function and then compare them to the boundary points $0$ and $1/L$.

Setting the subgradient of the objective to zero we arrive at the following equation:
\begin{align}\label{eq:stationary}
    \frac{2\sigma^2}{B_t}\cdot\frac{1}{(2-L\eta)^2}+ 
    2\Big(\gamma_t^2+2D_{\max}\gamma_t+ D_{\max}^2\ind(\eta<\eta_{t-1})\Big) \frac{L\eta-1}{(2\eta-L\eta^2)^2} = 0\,.
\end{align}
We consider the two cases below:
\begin{itemize}
    \item $\eta\ge\eta_{t-1}$: In this case, \eqref{eq:stationary} reduces to
    \[
    \frac{\sigma^2}{B_t}+ 
    \Big(\gamma_t^2+2D_{\max}\gamma_t\Big) \frac{L\eta-1}{\eta^2} = 0\,,
    \]
    which is a quadratic equation in $\eta$. Solving for $\eta$, the positive solution is given by $\tau_1$~\eqref{eq:tau1}. This case happens only when the solution satisfies the condition of the case, namely $\eta_{t-1}\le \tau_{1,t}$.
    \item $\eta\le\eta_{t-1}$. In this case, \eqref{eq:stationary} reduces to
    \[
    \frac{\sigma^2}{B_t}+ 
    \Big(\gamma_t^2+2D_{\max}\gamma_t+ D_{\max}^2\Big) \frac{L\eta-1}{\eta^2} = 0\,,
    \]
    which admits the positive solution $\tau_{2,t}$~\eqref{eq:tau2}. This case happens only when the solution satisfies the condition of the case, namely $\tau_{2,t}\le\eta_{t-1}$.
\end{itemize}
If $\tau_{1,t}<\eta_{t-1}<\tau_{2,t}$, then in both of the above cases, the solution happens at the boundary value $\eta_{t-1}$. This brings us to the following characterization for $\eta^*_t$:
\begin{align}
    \eta^*_t = \begin{cases}
    \tau_{1,t} & \text{if } \eta_{t-1}^* \le \tau_{1,t},\\
    \eta_{t-1}^* & \text{if } \tau_{1,t} \le \eta_{t-1}^* \le \tau_{2,t}\\
    \tau_{2,t} & \text{if } \eta_{t-1}^* \ge \tau_{2,t}\,.
    \end{cases}
\end{align}
Note that the above characterization was based on the stationary points of the objective. we next examine if the above solution satisfies the boundary conditions. Obviously $\eta_t^*>0$. We also claim that $\eta^*_t\le 1/L$. For this, we only need to show that $\tau_{2,t}\le 1/L$ (because $\eta^*_t\le \tau_{2,t}$ for all values of $\eta_{t-1}$). Invoking definition of $\tau_{2,t}$, we have 
\[
\tau_{2,t}:= \frac{B_t}{2\sigma^2}\left(\sqrt{b_{2,t}^2L^2+\frac{4\sigma^2}{B_t} b_{2,t}} - b_{2,t}L\right), \quad b_{2,t} := (\gamma_t+D_{\max})^2\,.
\]
It is easy to see that $\tau_{2,t}\le 1/L$ follows simply from $b_{2,t}^2L^2+4\frac{\sigma^2}{B_t}b_{2,t} < (\frac{2\sigma^2}{LB_t}+ b_{2,t} L)^2$.
%===================================
\subsection{Proof of Theorem~\ref{thm:le-convex-reg}}

Recall that
\[
    \delta_t: = \nabla\ell^B_t(\theta_t) -\nabla \ellbar_t(\theta_t)\,,
\]
as the difference between the gradient of the expected loss (at step $t$) and the gradient of the batch average loss at that step. Writing $D_{t+1}$ in terms of this notation, we get
\begin{align}\label{eqn:common_}
D_{t+1}^2 
&= D_t^2 - 2\eta_t\<\nabla \ellbar_t(\theta_t)+\delta_t, \theta_t - \theta_t^*\> +\|v_t\|^2
+2\<v_t, (\theta_t-\eta_t \nabla\ell^B_t(\theta_t))-\theta^*_t\>\nonumber\\ &\quad +\eta_t^2\Big(\|\nabla \ellbar_t(\theta_t)\|^2+ \|\delta_t\|^2+2\<\delta_t,\nabla\ellbar_t(\theta_t)\> \Big)\,.
\end{align}

Since the loss function $\ell(\theta,z)$ is $L$-smooth and $\mu$-strongly convex, the expected loss $\ellbar_t(\theta)$ is also $L$-smooth and $\mu$-strongly convex and by invoking \citet[Lemma 3($iii$)]{zhou2018fenchel},
we have
\[
\<\nabla \ellbar_t(\theta_t), \theta_t - \theta^*_t\>
\le \ellbar_t(\theta_t) - \ellbar_t(\theta^*_t) + \frac{1}{2\mu}\|\nabla \ellbar_t(\theta_t)\|^2\,.
\]
Using this bound in \eqref{eqn:common_}, we obtain
\begin{align}\label{eqn:common_1}
D_{t+1}^2 
&\ge D_t^2 - 2\eta_t(\ellbar_t(\theta_t)- \ellbar_t(\theta^*_t)) +\|v_t\|^2
+2\<v_t, (\theta_t-\eta_t \nabla\ell^B_t(\theta_t))-\theta^*_t\>\nonumber\\ &\quad +\Big(\eta_t^2 - \frac{\eta_t}{\mu}\Big)\|\nabla\ellbar_t(\theta_t)\|^2+\eta_t^2\|\delta_t\|^2- 2\eta_t \<\delta_t,\theta_t - \theta_t^*-\eta_t\nabla \ellbar_t(\theta_t)\>.
\end{align}
We next use \citet[Lemma 4, item 5]{zhou2018fenchel} and the fact that $\nabla\ellbar_t(\theta^*_t) = 0$ to get
\begin{align}
    \|\nabla\ellbar_t(\theta_t)\|^2\le 2L(\ellbar_t(\theta_t)-\ellbar_t(\theta^*_t))\,.
\end{align}
Using the above bound into \eqref{eqn:common_1}, for $\eta_t\le 1/\mu$,
we obtain
\begin{align}\label{eqn:common_1}
D_{t+1}^2 
&\ge D_t^2 - 2\Big(\eta_t+\frac{L}{\mu}\eta_t - \eta_t^2L\Big)(\ellbar_t(\theta_t)- \ellbar_t(\theta^*_t)) +\|v_t\|^2
+2\<v_t, (\theta_t-\eta_t \nabla\ell^B_t(\theta_t))-\theta^*_t\>\nonumber\\ &\quad +\eta_t^2\|\delta_t\|^2- 2\eta_t \<\delta_t,\theta_t - \theta_t^*-\eta_t\nabla \ellbar_t(\theta_t)\>\,.
\end{align}
We recognize that $\theta_t-\eta_t \nabla\ell^B_t(\theta_t)=\theta_{t+1}$ by the SGD update, and let $a'_t: = 2(\eta_t+\frac{L}{\mu}\eta_t - \eta_t^2L)$, with $\eta_t \le 1/\mu$.

Next we obtain a telescoping series for $\Reg(T)$ as before. Continuing as before (in Theorem~\ref{thm:convex-reg}), we can (1) isolate $\ellbar_t(\theta_t)-\ellbar_t(\theta^*_t)$ on the left-hand side, and (2) take expectations: first conditioned on the filtration $\bz_{[t-1]}$ and then an unconditioned expectation, to get:
\begin{align*}
&\E[\Reg(T)]= \sum_{t=1}^T \E[\reg_t]
\ge \mathbb{E}\left[\sum_{t=1}^T \left(\frac{D_{t}^2}{a'_{t}} - \frac{D_{t+1}^2}{a'_{t+1}}\right) +  \frac{\sigma^2\eta_t^2}{B_ta'_t} + \frac{\|v_t\|^2}{a'_t} + 2\frac{\<v_{t}, \theta_{t+1}-\theta^*_{t}\>}{a'_{t}}\right], 
\end{align*}
which completes the proof of theorem.

\section{Proof of theorems and technical lemmas for non-convex loss}
\subsection{Proof of Theorem~\ref{thm:non-Nconvex-reg}}
% We consider the function 
% \[
% \mathcal{C}_t(\theta) := \eta_t\<\nabla\ell_t(\theta_t),\theta\>+\frac{1}{2}\|\theta - \theta_t\|^2\,. 
% \]
% By the projected SGD update rule, we have $\theta_{t+1} = \arg\min_{\theta\in\Theta} \mathcal{C}_t(\theta)$. Note that $\mathcal{C}_t(\theta)$ is quadratic and so convex in $\theta$ (even if $\nabla \ell_t$ may be non-convex function, here it is evaluated at $\theta_t$.) By convexity of $\mathcal{C}_t(\theta)$ and optimality of $\theta_{t+1}$, we have
% \[
% \<\theta-\theta_{t+1},\nabla\mathcal{C}_t(\theta_{t+1})\>\ge0, \quad \forall \theta\in\Theta\,.
% \]
% Using the above condition at $\theta=\theta_t$, and expanding $\nabla\mathcal{C}_t(\theta_{t+1})$, we obtain
% \begin{align*}
% \<\theta_t-\theta_{t+1},\eta_t\nabla\ell_t(\theta_t)+\theta_{t+1} - \theta_t\>\ge 0\,,
% \end{align*}
% which can be rewritten as
% \begin{align}\label{eq:cond-nabla}
% \<\theta_{t+1}-\theta_{t},\nabla\ell_t(\theta_t)\>
% \le -\frac{1}{\eta_t}\|\theta_t-\theta_{t+1}\|^2\,.
% \end{align}
We note that by Assumption~\ref{ass:main},
\begin{align}\label{eq:L-smooth}
    \Big|\ellbar_{t}(\theta_{t+1}) - \ellbar_{t}(\theta_{t})
    -\<\nabla \ell_t(\theta_t),\theta_{t+1}-\theta_t\>\Big|
\le \frac{L}{2}\|\theta_{t+1}-\theta_t\|^2= \frac{L}{2} \eta_t^2\|\nabla\ell_t^B(\theta_t)\|^2\,.
\end{align}
Therefore,
\begin{align*}
    \ellbar_{t}(\theta_{t+1}) - \ellbar_{t}(\theta_{t}) &\le \<\nabla \ellbar_t(\theta_t),\theta_{t+1}-\theta_t\> + \frac{L}{2}\eta_t^2\|\nabla\ell_t^B(\theta_t)\|^2\\
    &\le -\eta_t\<\nabla \ellbar_t(\theta_t), \nabla \ell^B_t(\theta_t)\>+\frac{L}{2}\eta_t^2\|\nabla\ell_t^B(\theta_t)\|^2.
\end{align*}
Recall the notation $\delta_t := \nabla\ell_t^B(\theta_t)- \nabla \ellbar_t(\theta_t)$, by which we get
\begin{align*}
\ellbar_{t}(\theta_{t+1}) - \ellbar_{t}(\theta_{t}) &\le
-\eta_t\|\nabla \ellbar_t(\theta_t)\|^2-\eta_t\<\nabla\ellbar_t(\theta_t),\delta_t\> + \frac{L}{2}\eta_t^2 \left(\|\nabla\ellbar_t(\theta_t)\|^2+2\<\nabla\ellbar_t(\theta_t),\delta_t\> + \|\delta_t\|^2\right)\\
&= -\left(\eta_t - \frac{L}{2}\eta_t^2\right)\|\nabla \ellbar_t(\theta_t)\|^2 - (\eta_t - L\eta_t^2) \<\nabla\ellbar_t(\theta_t),\delta_t\> + \frac{L}{2}\eta_t^2 \|\delta_t\|^2\,.
\end{align*}
By condition $\eta_t\le 1/L$, we have $a_t = \eta_t - L\eta_t^2 >0$. Rearranging the terms in the above inequality, we obtain
\begin{align}\label{eq:nconv2}
    \|\nabla \ellbar_t(\theta_t)\|^2 \le 2\frac{\ellbar_t(\theta_t) - \ellbar_t(\theta_{t+1})}{a_t} -2\<\nabla\ellbar_t(\theta_t),\delta_t\> + \frac{L\eta_t^2}{a_t} \|\delta_t\|^2\,.
\end{align}
Since $\theta_t, \theta^*_t,\eta_t$ are adapted to the filtration $\bz_{[t-1]}$, we have
\[
\E[\<\nabla\ellbar_t(\theta_t),\delta_t\>|\bz_{[t-1]}] = 
\<\nabla\ellbar_t(\theta_t),\E[\delta_t\>|\bz_{[t-1]]} = 0\,.
\]
Therefore, by taking expectation from the both sides of \eqref{eq:nconv2}, first conditional on $\bz_{[t-1]}$ and then with respect to $\bz_{[t-1]}$ we get
\begin{align}\label{eq:nconv2}
    \E[\|\nabla \ellbar_t(\theta_t)\|^2] &\le 2\frac{\ellbar_t(\theta_t) - \ellbar_t(\theta_{t+1})}{a_t}  + \frac{L\eta_t^2}{a_t} \frac{\sigma^2}{B_t}\nonumber\\
    &\le 2\frac{\ellbar_t(\theta_t) - \ellbar_{t+1}(\theta_{t+1}) }{a_t} + \frac{|\ellbar_{t+1}(\theta_{t+1})-\ellbar_t(\theta_{t+1})|}{a_t} + \frac{L\eta_t^2}{a_t} \frac{\sigma^2}{B_t}\nonumber\\
    &= 2\frac{\ellbar_t(\theta_t) - \ellbar_{t+1}(\theta_{t+1}) }{a_t} + \frac{\gamma_t}{a_t} + \frac{L\eta_t^2}{a_t} \frac{\sigma^2}{B_t}\,.
\end{align}
Summing both sides over $t=1,\dotsc, T$, we have
\begin{align*}
\E[\Reg(T)] &=\sum_{t=1}^T \E[\|\nabla \ellbar_t(\theta_t)\|^2]\\ 
&\le \sum_{t=1}^T\E\left[ \left(\frac{2\ellbar_t(\theta_t)}{a_t} - \frac{2\ellbar_{t+1}(\theta_{t+1})}{a_{t}}\right) + L\frac{\sigma^2\eta_t^2}{B_ta_t} + \frac{\gamma_t}{a_t}\right]\\
    &= \sum_{t=2}^T \E\left[2\ellbar_t(\theta_t) \left(\frac{1}{a_t} - \frac{1}{a_{t-1}}\right)\right]
    +\E\left[\frac{2\ellbar_1(\theta_1)}{a_1}-\frac{2\ellbar_{T+1}(\theta_T)}{a_{T+1}}\right] + \sum_{t=1}^T \E\left[L\frac{\sigma^2\eta_t^2}{B_ta_t} + \frac{\gamma_t}{a_t} \right]\,.
\end{align*}
The result follows by noting that $\ellbar_{T+1}(\theta_{T+1})\ge 0$.

% Combining Equations~\eqref{eq:cond-nabla} and \eqref{eq:L-smooth}, we arrive at
% \begin{align}
%     \ell_{t+1}(\theta_{t+1})\le \ell_{t}(\theta_{t})+\Big(\frac{L}{2}-\frac{1}{\eta_t}\Big)\|\theta_t-\theta_{t+1}\|^2
%     \le \ell_{t}(\theta_{t})+\Big(\frac{L}{2}-\frac{1}{\eta_t}\Big) \|\nabla\ell_t(\theta_t)\|^2\,,\label{eq:nconv1}
% \end{align}
% where in the last step we used the fact that projection onto a convex set is contraction and so
% \[
% \|\theta_{t+1}-\theta_t\| = \|\Pi_\Theta(\theta_t -\eta_t\nabla \ell_t^B(\theta_t) - \Pi_\Theta(\theta_t) )\|
% \le \|\theta_t -\eta_t\nabla \ell_t^B(\theta_t) - \theta_t\|
%  = \eta_t\|\nabla \ell_t^B(\theta_t)\|
% \]
\section{Auxiliary lemmas}\label{app:lem}
\begin{lemma}\label{lem:stein}
Let $x\in\reals^d$ such that $x\sim \normal(0,I_d)$.
For any fixed vector $u\in\reals^d$, we have
\[
    \E[(xx^\sT-I) uu^\sT (xx^\sT - I)^\sT] = \|u\|^2 I + uu^\sT\,.
\]
\end{lemma}
\begin{proof}
By Stein's lemma, for any function $g:\reals^d \to \reals$ we have
\[
\E[(xx^\sT - I)g(x)] = \E[\nabla^2g(x)]\,.
\]
Using the above identity with $g(x) = (u^\sT x)^2$ we obtain
\begin{align}
    \E[xx^\sT(u^\sT x)^2] = 2uu^\sT + \|u\|^2 I\,.
\end{align}
Using the above characterization, we get
\begin{align*}
 \E[(xx^\sT-I) uu^\sT (xx^\sT - I)^\sT] &=
 \E[xx^\sT (u^\sT x)^2] - u(u^\sT x) x^\sT - x(x^\sT u) u^\sT + uu^\sT\\
 &=2uu^\sT + \|u\|^2 I - 2 uu^\sT+ uu^\sT\\
 &= uu^\sT+ \|u\|^2 I\,,
\end{align*}
which completes the proof.
\end{proof}

We next present Ito's lemma, which allows to find the differential of a time-dependent function of a stochastic process. 
\begin{lemma}[It\^o's lemma,~\cite{oksendal2013stochastic}]
\label{lem:Ito}
Let $X_t\in\reals^p$ be a vector of  It\^o drift-diffusion process, such that
\[
    \de X_t = f(t,X_t) \de t+ g(t,X_t) \de W_t\,,
\]
with $W_t$ being an $q$-dimensional standard Brownian motion and $f(t,X_t)\in\reals^p$ and $g(t,X_t)\in\reals^{p\times q}$. Consider a scalar process $Y(t)$ defined by $Y(t) = \phi(t, X(t))$, where $\phi(t, X)$ is a scalar function which is continuously differentiable with respect to $t$ and twice continuously differentiable with respect to $X$. We then have
\begin{align*}
    \de Y_t &= \tilde{f}(t,X_t) \de t+ \tilde{g}(t,X_t) \de W_t\,,\\
    \tilde{f}(t,X_t)&= \phi_t(t,X_t) + \phi_x(t,X_t)^\sT f(t,X_t) +\frac{1}{2}{\rm tr}\left(g(t,X_t)^\sT \phi_{xx}(t,X_t)g(t,X_t)\right)\\
    \tilde{g}(t,X_t)&=  \phi_x(t,X_t)^\sT g(t,X_t)\,.
\end{align*}
\end{lemma}

\end{document}